\documentclass[accepted]{uai2023} 

\usepackage{natbib} 
\usepackage[american]{babel}

\author[1]{\href{mailto:<pierreglaser@gmail.com>?Subject=Your UAI 2023 paper}{Pierre~Glaser}{}}
\author[2]{David~Widmann}
\author[3]{Fredrik~Lindsten}
\author[1]{Arthur~Gretton}
\affil[1]{%
    Gatsby Computational Neuroscience Unit\\
    University College London\\
    London, UK
}
\affil[2]{%
    Department of Information Technology\\
    Uppsala University\\
    Sweden
}
\affil[3]{%
    Division of Statistics and Machine Learning\\
    Linköping University\\
    Sweden
}

\title{Fast and Scalable Score-Based Kernel Calibration Tests}

\usepackage{mathtools} 
\usepackage{booktabs} 
\usepackage{tikz} 

\usepackage{amssymb}
\usepackage{amsthm}
\usepackage{amsfonts,bm,bbm}
\usepackage[algoruled,linesnumbered,hanginginout,algo2e]{algorithm2e}
\usepackage[capitalize,noabbrev]{cleveref}
\usepackage{preamble}

\usepackage{mathtools} 
\usepackage{booktabs} 
\usepackage{tikz} 

\usepackage{preamble}
\usepackage{chngcntr}
\usepackage{placeins}

\crefname{assump}{assumption}{assumptions}

\theoremstyle{plain}
\newtheorem{theorem}{Theorem}[section]
\newtheorem{proposition}{Proposition}[section]
\newtheorem{lemma}[theorem]{Lemma}
\newtheorem*{lemma-non}{Lemma}

\newtheorem*{theorem-non}{Theorem}

\theoremstyle{definition}
\newtheorem{definition}[theorem]{Definition}

\theoremstyle{remark}

\SetKwInput{Parameters}{Parameters}

\SetCommentSty{mycommfont}

\makeatletter
\newcommand{\DeclareMathOperatorIfNotDefined}[2]{%
  \@ifundefined{#1}{%
    \expandafter\DeclareMathOperator\csname #1\endcsname{#2}%
  }{}%
}
\makeatother

\makeatother
\DeclareMathOperatorIfNotDefined{E}{\mathbb{E}}

\begin{document}
\maketitle
\begin{abstract}
We introduce the \emph{Kernel Calibration Conditional Stein Discrepancy test} (KCCSD test), a non-parametric, kernel-based test for assessing the calibration of probabilistic models with well-defined scores. In contrast to previous methods, our test avoids the need for possibly expensive expectation approximations while providing control over its type-I error. We achieve these improvements by using a new family of kernels for score-based probabilities that can be estimated without probability density samples, and by using a conditional goodness-of-fit criterion for the KCCSD test's U-statistic. We demonstrate the properties of our test on various synthetic settings.
\end{abstract}

\section{Introduction}\label{sec:introduction}

Calibration is a statistical property of predictive probabilistic models that ensures that a model's prediction matches the conditional distribution of the predicted variable given the prediction.
A calibrated model expresses the uncertainty about its predictions reliably by being neither over- nor underconfident, and hence can be useful even if its accuracy is suboptimal.
This property is essential in safety-critical applications such as autonomous driving.
Unfortunately, empirical studies revealed that popular machine learning models such as deep neural networks tend to trade off calibration for accuracy~\citep{Guo2017}.
This has lead to an increased interest in the study of calibrated models in recent years.

Calibration has been studied in the meteorological and statistical literature for many decades~\citep[e.g.,][]{Murphy1977,DeGroot1983}.
For a long time, research on calibration has been focused on different notions of calibration for probabilistic classifiers~\citep[e.g.,][]{Murphy1977,DeGroot1983,plattProbabilisticOutputsSupport2000,zadrozny2001obtaining,Broecker2009,Naeini2015,Guo2017,kullBetaCalibrationWellfounded2017,kumar18train,kullTemperatureScalingObtaining2019,vaicenavicius2019evaluating,widmann2019calibration} and on calibration of quantiles and confidence intervals for real-valued regression problems~\citep[e.g.,][]{Ho2005,Rueda2006,Taillardat2016,songDistributionCalibrationRegression2019,Fasiolo2020}.
Regarding the calibration of classification models, different hypothesis tests have been proposed~\citep[e.g.,][]{coxTwoFurtherApplications1958,Broecker2007,vaicenavicius2019evaluating,widmann2019calibration,Gweon2022,leeTCalOptimalTest2022}.
Given a predictive model and a validation dataset, these tests output whether a model is likely to be uncalibrated.
The recent work of \citet{widmann2022calibration} generalized the calibration-framework introduced for classification in~\citep{widmann2019calibration} to (possibly multi-dimensional) continuous-valued predictive models.
In particular, \citet{widmann2022calibration} introduced a kernel-based hypothesis test for such general classes of models.

An important potential consumer of calibration tests is Bayesian inference, and in particular
simulation-based inference (SBI), for which miscalibration is particularly undesirable.
SBI~\citep{cranmer2020frontier} lies at the intersection of machine learning and domain sciences, and
refers to the set of methods that train probabilistic models to estimate the posterior over scientific parameters of interest given some observed data. The models are trained using pairs composed of parameters drawn from a prior distribution, and their associated ``synthetic'' observed data, obtained by running a probabilistic program called the \emph{simulator}, taking a parameter value as input, and that faithfully mimics the physical generative process of interest. The increasing number of use cases combined with advances in probabilistic modeling has elevated SBI to a critical role in solving complex scientific problems such as particle physics~\citep{gilman2018probing} and neuroscience~\citep{glockler2022variational, glaser2022maximum}. However, as discussed in~\citep{Hermans2021}, overconfidence in SBI models can conceal credible alternative scientific hypotheses, and result in incorrect discoveries~\citep{Hermans2021}, highlighting the need for principled and performant calibration tests suitable for such models.

While the theoretical framework of~\citet{widmann2022calibration} describes the calibration of any probabilistic model, applying its associated calibration test to Bayesian inference remains challenging: indeed, the test statistics require computing expectations against the probabilistic models of interest, for reasons bearing both to the calibration setting, and to the limitations of currently available  
kernel-based tools for probabilistic models.
Although such expectations can be computed exactly for classification models, expectations against generic probabilistic models are usually intractable and must be approximated. 
In cases where the models are \emph{unnormalized}, these approximations are both computationally expensive---sometimes, prohibitively---and biased,
thereby compromising theoretical guarantees of the calibration tests of~\citet{widmann2022calibration}, including type-I error control.

{\bfseries Contributions} 
In this paper, we introduce the kernel calibration-conditional Stein discrepancy (or KCCSD) test, a new nonparametric, score-based kernel calibration test which addresses the limitations of existing methods.
The KCCSD test builds on the insight that the definition of calibration given by~\citet{vaicenavicius2019evaluating} is a conditional goodness of fit property, as we remark in \cref{sec:calibration-as-cgof}. This fact allows us to leverage the kernel conditional goodness of fit  test proposed by~\citet{jitkrittum2020testing} as the backbone of the KCCSD test. Unlike the test-statistics of~\citet{widmann2022calibration}, the KCCSD test statistic does not contain \emph{explicit} expectations against the probabilistic models; however, as in~\citep{widmann2022calibration}, it requires evaluating a kernel between probabilities densities, which in most cases of interest introduces an (intractable) expectation against the densities. To eliminate this limitation, we construct two new kernels between probability distributions that do not involve expectations against its input distributions, while remaining suitable for statistical testing. These kernels rely on a generalized version of the Fisher divergence and are of independent interest. We investigate a connection between these kernels and diffusion, akin
to Stein methods, and discuss the relationships with other kernels on distributions.
By using such kernels in the KCCSD test statistic, we obtain a fast and scalable calibration test that remains consistent and calibrated for unnormalized models, answering the need for such tests discussed above. We confirm in \cref{sec:experiments} the properties and benefits of the KCCSD test against alternatives on synthetic experiments.

\section{Background}\label{sec:background}

\paragraph{Notation}
We consider probabilistic systems characterized by a joint distribution $\mathbb{P}(X, Y)$ of random variables $(X, Y)$ taking values in $\mathcal{X} \times \mathcal{Y}$,
and study \emph{probabilistic models} $P_{|\cdot} \colon x \in \mathcal{X} \longmapsto P_{|x}(\cdot) \in \mathcal{P}(\mathcal{Y})$ approximating the unknown conditional distribution of $Y$ given $X = x$, $P_{|x}(\cdot) \simeq \mathbb{P}(Y \in \cdot \,|\, X = x)$.
The target variable $Y$ is typically a parameter of a probabilistic system of interest---like synaptic weights in biological neural networks---while the input variable $X$ is observed data---like neuron voltage traces measured using electrophysiology.

\subsection{Calibration of Predictive Models}\label{subsec:calibration}

{\bfseries Calibration: General Definition}
A probabilistic model $P_{|\cdot}$ is called calibrated or reliable~\citep{brocker2008some,vaicenavicius2019evaluating,widmann2022calibration} if it satisfies
\begin{equation}\label{eq:calibration}
P_{|X} = \mathbb{P}\left(Y \in \cdot \mid P_{|X} \right) \qquad \mathbb{P}(X)\text{-a.s.}.
\end{equation}
Note that this definition applies to general predictive probabilistic models, also beyond classification, and only assumes that the conditional distributions on the right-hand side exist.

{\bfseries Hypothesis Testing: Kernel Calibration Error}  There are multiple ways to test whether a given predictive probabilistic model is calibrated.
In this section, we introduce the kernel-based tests of~\citet{widmann2019calibration} and their later generalization~\citep{widmann2022calibration}, since our KCCSD test is built on these approaches.
These tests turn the equality between conditional distributions present in \cref{eq:calibration} into a more classical equality between two joint distributions.
The transformation is achieved by noting that
\begin{multline*}
P_{|X} = \mathbb{P}\left(Y \in \cdot \mid P_{|X} \right) \quad \mathbb{P}(X)\text{-a.s.} \\
\iff (P_{|X}, Y) \stackrel{d}{=} (P_{|X}, Z)
\end{multline*}
where $Z$ is an ``auxiliary'' variable such that $Z \,|\, P_{|X} \sim P_{|X}(\cdot)$.
This identity was used by \citet{widmann2022calibration} to construct an MMD-type calibration test based on the (squared) kernel calibration error~(SKCE) criterion
\begin{equation}\label{eq:SKCE}
\sup_{h \in \mathcal  B(0_\mathcal{H}, 1)} \E_{(x, y, z) \sim \mathbb{P}(X, Y, Z)} \big[h(P_{|x}, y) - h(P_{|x}, z) \big].
\end{equation}
 Here, $\mathcal{B}(0_{\mathcal{H}}, 1)$ is the unit ball of a reproducing kernel Hilbert space~(RKHS) $\mathcal H$ of functions with positive definite kernel $k_{\mathcal{H}} \colon (P_ {|\mathcal{X}} \times \mathcal{Y})^2  \to \mathbb{R}$. As noted by \citet{widmann2022calibration}, the SKCE generalizes the (squared) kernel classification calibration error~(SKCCE) defined for the special case of discrete output spaces $\mathcal{Y} = \left \{ 1, \dots, d \right \}$~\citep{widmann2019calibration}, to continuous ones. Given $n$ pairs of samples ${\{(P_{|x^i}, y^i)\}}_{i=1}^n \stackrel{\text{i.i.d.}}{\sim} \mathbb{P}(P_{|X}, Y)$, \citet{widmann2022calibration} consider the following SKCE estimator
\begin{equation} \label{eq:SKCE-u-statistics-estimator}
    \widehat{\operatorname{SKCE}} = \frac{2}{n(n-1)} \sum\limits_{1 \leq i < j \leq n}G((P_{|x^i}, y^i), (P_{|x^j}, y^j))
\end{equation}
where
\begin{equation}\label{eq:G}
\begin{split}
    G((p, y), (p', y')) \coloneqq & k((p, y), (p', y'))\\
                                  &- \E_{z \sim p} k((p, z), (p', y'))\\
                                  &- \E_{z' \sim p'} k((p, y), (p', z'))\\
                                  &+ \E_{z \sim p} \E_{z' \sim p'} k((p, z), (p', z')).
\end{split}
\end{equation}
For a target false rejection rate $\alpha \in (0, 1)$, the test of \citet{widmann2022calibration}
follows standard methodology in recent nonparametric testing \citep{gretton2012kernel, gretton2007kernel, Chwialkowski16KGOF} by rejecting the null hypothesis that the model is calibrated if $\widehat{\operatorname{SKCE}}> \gamma_{1-\alpha}$, where $\gamma_{1-\alpha}$ denotes the $(1-\alpha)$-quantile of $\widehat{\operatorname{SKCE}}$ under the null. While various methods are available to estimate this quantile, all tests experiments in this paper use a \emph{bootstrap} approach \citep{arcones1992bootstrap}. As discussed by \citet{widmann2022calibration}, \cref{eq:SKCE-u-statistics-estimator} contains two important sources of possible intractability:

{\bfseries First Problem}
The last three terms in the sum are expectations under predictions of the probabilistic model of interest.
However, closed-form expressions for these expectations are only available in restricted cases, such as for classification and for Gaussian models coupled with Gaussian kernels.
When these expectations are not available, they must be approximated numerically.
If the distributions $P_{|X}$ are given in the form of unnormalized models, this approximation requires running expensive approximation methods that often take the form of an MCMC algorithm and must be performed for every sample of $P_{|X}$ used to estimate the test statistic.

{\bfseries Second Problem}
The second source is the evaluation of the kernel function $k$.
We restrict our attention to the conventional form of tensor-product type kernels $k((p, y), (p', y')) = k_{P}(p, p') k_{Y}(y, y')$ chosen in this setting.
While typically many tractable choices for the kernel $k_{Y}$ exist (taking as input discrete or Euclidean values), the choices for $k_{P}$, taking as input two probability distributions $p$ and $p'$, are more limited and require expensive approximations methods when working with unnormalized models.

A popular approach to design kernels on distributions~\citep{DBLP:conf/aistats/0001GPS15, DBLP:journals/jmlr/SzaboSPG16} is to first embed the probability distributions in a Hilbert space $\mathcal{H}$ using a map $\phi$, and then compose it with a kernel $k_{\mathcal{H}}$ on $\mathcal{H}$:
\begin{equation*}
k_{P}(p, p') = k_{\mathcal{H}}(\phi(p), \phi(p')).
\end{equation*}
Any valid  kernel on $ \mathcal{H}$, like the linear kernel $k_{\mathcal{H}}(z, z') = \left \langle z, z' \right \rangle_{\mathcal{H}}$, the Gaussian kernel $k_{\mathcal{H}}(z, z') = e^{-\left \|z - z'\right \|^2_{\mathcal{H}}}$, or the inverse multiquadric kernel $k_{\mathcal{H}}(z, z') = (1 + \left \|z - z'\right\|^2_{\mathcal{H}})^{-1}$ can be used.
In practice, the map $\phi$ can be set to be the \emph{mean embedding} map to an RKHS $\mathcal{H}$, e.g., $\phi(\mu) = \int k_{\mathcal{H}}(z, \cdot) \,\mu(\mathrm{d}z)$.
Kernels $k_{\mathcal{H}}$ that are functions of $\left\| \phi(\mu) - \phi(\nu) \right \|^2_{\mathcal{H}} \coloneqq \operatorname{MMD}^2(\mu, \nu)$, are often referred to as MMD-type kernels~\citep{DBLP:conf/icml/MeunierPC22}.
Other distances, like the Wasserstein distance in 1 dimension or the sliced Wasserstein distance~\citep{DBLP:journals/jmiv/BonneelRPP15} in multiple dimensions, also take this form for some choice of $\phi$ and $\mathcal{H}$, and can thus be used to construct kernels on distributions~\citep{DBLP:conf/icml/MeunierPC22}.
In general, however, computing $k_{P}(p, p')$ becomes intractable apart from special cases such as when $p$ and $p'$ are Gaussian distributions.
While there exist finite-samples estimators for such kernels, a fast calibration estimation method based on \cref{eq:SKCE} would require an estimator that does not require samples from $p$ and $p'$.

\subsection{Kernel Conditional Goodness-of-Fit Test}

We briefly introduce the background on goodness-of-fit methods relevant to our new test.
\emph{Conditional goodness-of-fit} (or CGOF) testing adapts the familiar goodness of fit tests
to the conditional case. In particular, CGOF tests whether
\begin{equation}\label{eq:cgof-hypothesis}
H_0\colon Q_{|Z} = \mathbb{P}(Y \in \cdot \,|\, Z) \qquad \mathbb{P}(Z)\text{-a.s.}
\end{equation}
given a candidate $Q_{|z}$ for the conditional distribution $\mathbb{P}({Y \in \cdot} \,|\, Z=z)$ and samples ${\{(z^{i}, y^{i})\}}_{i=1}^n \stackrel{\text{i.i.d}}{\sim} \mathbb{P}(Z, Y)$.
This problem was studied by \citet{jitkrittum2020testing} for the case $\mathcal{Z} \times \mathcal{Y} \subset \mathbb{R}^{d_z} \times \mathbb{R}^{d_y}$ and models $Q_{|z}$ with a differentiable, strictly positive density $f_{Q_{|z}}$.
They proposed a kernel CGOF test for \cref{eq:cgof-hypothesis} based on the (squared) kernel conditional Stein discrepancy~(KCSD)
\begin{equation}\label{eq:cgof-metric}
    D_{Q_{|\cdot}}(\mathbb{P}) \coloneqq \left \| \mathbb{E}_{(z,y) \sim \mathbb{P}(Z, Y)}\left [ K_{z} \xi_{Q_{|z}}(y, \cdot) \right ]  \right \|^2_{\mathcal  F_{K}}
\end{equation}
Here, $\mathcal{F}_{K} $ is an $\mathcal{F}^{d_y}_{l}$ (e.g $\overbrace{\mathcal F_l \times \dots \times \mathcal F_l}^{d_y \text{ times}}$)-vector-valued RKHS with kernel
$K\colon \mathcal{Z} \times \mathcal{Z} \to \mathcal{L}(\mathcal{F}^{d_y}_{l}, \mathcal{F}^{d_{y}}_{l})$, $K_{z}$ is its associated linear operator on $\mathcal{F}^{d_y}_l$ with $K_{z}g \coloneqq K(z, \cdot)g \in  \mathcal{L}(\mathcal{Z}, \mathcal{F}^{d_y}_l)$ for $g\in \mathcal{F}^{d_y}_l$, $\mathcal{F}_{l} $ is an RKHS on $\mathcal{Y}$ with kernel $l \colon \mathcal{Y} \times \mathcal{Y} \to \mathbb{R}$
and $\xi_{Q_{|z}} $ is the ``kernelized score'':
\begin{equation*}
    \xi_{Q_{|z}}(y, \cdot) = l(y, \cdot)\nabla_{y} \log f_{Q_{|z}}(y) + \nabla_{y}  l(y, \cdot) \in \mathcal F_{l}^{d_y}.
\end{equation*}

We refer to \citet[Section 2 and 3]{jitkrittum2020testing} for an intuition behind the KCSD formula, and its relationship to the more familiar Kernel Stein Discrepancy \cite{Chwialkowski16KGOF, gorham2017measuring}. Under certain assumptions, the null hypothesis in \cref{eq:cgof-hypothesis} is true if and only if $D_{Q_{|\cdot}}(\mathbb{P}) = 0$. In particular, the latter will hold \citet[Theorem~1]{jitkrittum2020testing} if $ \mathcal  Y $ and $\mathcal Z$ are compact and the kernels $K$ and $l$ are universal, meaning that $\mathcal{F}_K$ (resp.\ $\mathcal F_l$) is dense with respect to $\mathcal C(\mathcal Z, \mathcal F^{d_y}_l)$ (resp.\ $C(\mathcal Y, \mathbb R)$), the space of continuous functions from $\mathcal Z$ to $\mathcal F^{d_y}_l$ (resp.\ $\mathcal{Y}$ to $\mathbb{R}$) \footnote{These statements hold for noncompact $\mathcal Y, \mathcal Z$ by replacing continuous functions by continuous functions vanishing at infinity \citep[Theorem 1]{jitkrittum2020testing}.}.
An instance of a universal $\mathcal F^{d_y}$-reproducing kernel is given by
\begin{equation}\label{eq:kernel_identity}
K(z, z') = k(z, z')I_{\mathcal{F}_l^{d_y}}
\end{equation}
where $I_{\mathcal{F}_l^{d_y}} \in \mathcal  L(\mathcal  F^{d_y}_l, \mathcal  F^{d_y}_{l})$ is the identity operator and $k$ is a real-valued universal kernel~\citep{carmeli2010vector}.
\citet{jitkrittum2020testing} showed that the CGOF statistic $D_{Q_{|\cdot}}(\mathbb{P})$
admits an unbiased consistent estimator and used it to construct hypothesis tests of \cref{eq:cgof-hypothesis} with operator-valued kernels of the form in \cref{eq:kernel_identity}.

\section{Kernel Calibration-Conditional Stein Discrepancy}\label{sec:calibration-as-cgof}

Calibration testing in the sense of \cref{eq:calibration} is an instance of \emph{conditional goodness-of-fit} testing of \cref{eq:cgof-hypothesis} with input $Z = P_{|X}$, target $Y$, and models $Q_{|z} = z = P_{|x}$.
Assuming that $\mathcal{Y} \subset \mathbb{R}^{d_y}$ and that distributions $P_{|x}$ have a differentiable, strictly positive density $f_{P_{|x}}$. In that case, the (squared) kernel conditional Stein discrepancy in \cref{eq:cgof-metric} becomes
\begin{equation}\label{eq:calibration-cgof-metric}
    C_{P_{|\cdot}}(\mathbb{P}) \coloneqq \left \| \mathbb{E}_{(x,y) \sim \mathbb{P}(X, Y)}\left [ K_{P_{|x}} \xi_{P_{|x}}(y, \cdot) \right ]  \right \|^{2}_{\mathcal  F_{K}},
\end{equation}
where now $K$ is a kernel on $P_{|\mathcal{X}}$.
To emphasize the calibration setting, we call $C_{P_{|\cdot}}$ the kernel calibration-conditional Stein discrepancy~(KCCSD).
Similar to the KCSD, given samples ${\{P_{|x^i}, y^i\}}_{i=1}^n \stackrel{\text{i.i.d.}}{\sim} \mathbb{P}(P_{|X}, Y)$ and assuming a kernel $K$ of the form in \cref{eq:kernel_identity}, statistic $C_{P_{|\cdot}}(\mathbb{P})$ has an unbiased consistent estimator
\begin{equation*}\label{eq:calibration-gof-estimator}
    \widehat{C_{P_{|\cdot}}} = \frac{2}{n(n - 1)} \sum_{1 \leq i < j \leq n} H((P_{|x^i}, y^i), (P_{|x^j}, y^j))
\end{equation*}
where
\begin{equation}\label{eq:calibration-cgof-test-statistic}
    H((p, y), (p', y')) \coloneqq k(p, p') h((p, y), (p', y'))
\end{equation}
with
\begin{equation}\label{eq:calibration-cgof-test-statistic-h}
\begin{split}
    &h((p, y), (p', y')) \coloneqq l(y, y') s_p(y)^{\top} s_{p'}(y') \\
    &\qquad+ \sum\limits_{i=1}^{d_y} \frac{\partial^2}{ \partial y_i \partial y_i'} l(y, y')
    + s_p(y)^{\top} \nabla_{y'} l(y, y') \\
    &\qquad+ s_{p'}(y')^{\top} \nabla_{y} l(y, y'),
\end{split}
\end{equation}
where $s_{p}(y) \coloneqq \nabla_y \log f_{p}(y)$ (resp.\ $s_{p'}(y)$) is the \emph{score} of $p$ (resp.\ $p'$).
In Section~A in the supplement we discuss how the formula of $\widehat{C}_{P_{|\cdot}}$ generalizes to operator-valued kernels that are not of the form in \cref{eq:kernel_identity}.

The above framing of the calibration problem conveniently avoids the first source of possible intractability present in the SKCE.
For instance, for Gaussian models the test statistic can be evaluated exactly for arbitrary kernels $l$ on $\mathcal{Y}$ whereas a closed-form expression of the SKCE is known only in the special case where $l$ is a Gaussian kernel.

\Cref{prop:kccsd-relation-skce} shows that the KCCSD can be viewed as a special case of the SKCE.
More generally, as shown in Section~B, the KCSD is a special form of the MMD.

\begin{proposition}[Special case of Lemma~B.1]\label{prop:kccsd-relation-skce}
Under weak assumptions (see~Lemma~B.1),
the KCCSD with respect to kernels $l \colon \mathcal{Y} \times \mathcal{Y} \to \mathbb{R}$ and $k \colon P_{|\mathcal{X}} \times P_{|\mathcal{X}} \to \mathbb{R}$ is equivalent to the SKCE with kernel $H \colon (P_{|\mathcal{X}} \times \mathcal{Y}) \times (P_{|\mathcal{X}} \times \mathcal{Y}) \to \mathbb{R}$ defined in \cref{eq:calibration-cgof-test-statistic}.
\end{proposition}

The full testing procedure is outlined in \cref{alg:ccgof-tractable}.
The computations can be performed with kernels $K$ of the form in \cref{eq:kernel_identity} or more general operator-valued kernels, but crucially the method requires that $K$ is tractable.
Thus for general models of probability distributions, such as energy-based models and other unnormalized density models, it remains to address the second source of intractability, namely to construct a kernel $K$ that can be evaluated efficiently.

\begin{algorithm2e}
    \DontPrintSemicolon
    \caption{CGOF Calibration Test (Tractable Kernel)}\label{alg:ccgof-tractable}
	\KwData{Pairs ${\{(P_{|x^i}, y^i)\}}_{i=1}^n \stackrel{\text{i.i.d.}}{\sim} \mathbb{P}(P_{|X}, Y)$}
	\KwResult{Whether to reject $H_0 \colon \text{``model is calibrated''}$}
    \Parameters{Number of data samples $n$, kernel $l \colon \mathcal{Y}^2 \to \mathbb{R}$, kernel $k \colon (P_{|\mathcal{X}})^2 \to \mathbb{R}$, level $\alpha$}
    \BlankLine
    \tcc{Estimate KCCSD using \cref{eq:calibration-cgof-test-statistic-h} or (A.1)}
    $\widehat{C} \leftarrow \frac{2}{n(n-1)} \sum\limits_{1 \leq i < j \leq n} H((P_{|x^i}, y^i),(P_{|x^j}, y^j))$\;
    \tcc{Use e.g.\ bootstrap~\citep{arcones1992bootstrap}}
    $\widehat{C}_{\alpha} \leftarrow$ approximate $(1- \alpha)$-quantile of $\widehat{C}$\;
    \eIf{$ \widehat{C} < \widehat{C}_{\alpha}$}{%
        \Return{Fail to reject $H_0$}%
    }{%
        \Return{Reject $H_0$}%
    }
\end{algorithm2e}

\section{Tractable Kernels for General Unnormalized Densities}\label{subsec:expgfd-kernel}

In this section, we introduce two kernels between (density-based) probability
distributions that admit unbiased estimates that neither require samples from
the said distributions nor require access their normalizing constant.
Crucially, the properties of these new kernels allow to extend the scope of calibration
tests to a more general setting, including Bayesian inference.

\paragraph{General Recipe} 
As in prior work on kernels for distributions~\citep{DBLP:conf/icml/MeunierPC22,DBLP:journals/jmlr/SzaboSPG16},
our proposed kernels take the form of exponentiated Hilbertian metrics
\begin{equation*}
k(p, q) = e^{- \left \| \phi(p) - \phi(q) \right \|^2_{H} / (2\sigma^2)}
\end{equation*}
for two probability densities $p$ and $q$, defined on some set $\mathcal X \subset \mathbb R^d$, 
where $ H $ is some separable Hilbert space, $ \phi \colon p \mapsto \phi(p) \in H $ is a
feature map, and $ \sigma > 0 $ is a bandwidth parameter.
Our contributions in this section consist in pairs of carefully
designed $ \phi $ and $ H $ that will allow approximating $ k $ easily.

\subsection{The Generalized Fisher Divergence (Kernel)}
Our starting point is the \emph{Fisher Divergence}~\citep{lyu2012interpretation,sriperumbudur2017density,hyvarinen2005estimation}, also known as the \emph{Relative Fisher Information}~\citep{otto2000generalization}, between two probability densities $ p $ and $ q $, which is given by
\begin{equation*}\label{eq:fd}
    \operatorname{FD}(p, q) \coloneqq \int_{\mathcal  X} \left \| s_p(x) - s_q(x)  \right \|^2 p(x) \,\mathrm{d}x.
\end{equation*}
The Fisher Divergence is a convenient tool to compare unnormalized densities of the form
\begin{equation*}
p(x)\coloneqq \frac{ \overbrace{f(x)}^{\text{tractable}} }{ \underbrace{Z_f}_{\text{intractable}} } \quad \text{where} \quad Z_f\coloneqq \int_{\mathcal  X} f(x) \,\mathrm{d}x
\end{equation*}
as the score of $ p $ can be evaluated without knowing $ Z_f $:
\begin{equation*}
    s_p(x) = \nabla_{x} (\log f(x)/{Z_f}) = \nabla_{x} \log f(x).
\end{equation*}
This property confers to the (squared) Fisher Divergence a tractable unbiased
estimator given $ n $ i.i.d.\ samples $ \{X^{i}
\}_{i=1}^{n} $ from $ p $, which takes the form:
\begin{equation*}
    \widehat{\operatorname{FD}(p, q) } = \frac{1}{n} \sum_{i=1}^{n} \|  s_p(X^{i}) - s_q(X^{i})  \|^2.
\end{equation*}
While the assumption ensuring access to samples from $ p $ is realistic in the unsupervised learning literature~\citep{hyvarinen2005estimation}, or when dealing with special instances of unnormalized densities such as truncated densities $ f(x) = p(x)\mathbf{1}_{x \in \mathcal  C}  $, it does not hold in the context of unnormalized models.
We overcome this issue by constructing a generalized version of the Fisher Divergence:
\begin{definition}[Generalized Fisher Divergence]
    Let $ p $,  $ q $ be two probability densities on $ \mathcal  X $, and
    $ \nu $ a probability measure on $ \mathcal  X $. The \emph{Generalized
    Fisher Divergence} between $ p $ and $ q $ is defined as
\begin{equation*} \label{eq:gfd}
\operatorname{GFD}_{\nu}(p, q) \coloneqq \int_{\mathcal  X} \left \| s_p(x) - s_q(x)  \right \|^2 \,\nu(\mathrm{d}x)
\end{equation*}
if $ \mathbb{E}_{ \nu }\left \| s_p  \right \|^2, \mathbb{E}_{ \nu }\left \|
s_q  \right \|^2 < +\infty $, and $ +\infty $ otherwise.
\end{definition}
The Generalized Fisher Divergence differs from the Fisher Divergence in that
the integration is performed with respect to some given base measure $ \nu $
instead of $ p $. If the support of $ \nu $ covers the support of $ p $ and $ q
$,  then we have that $ \operatorname{GFD}_{\nu}(p, q) = 0 $ iff.\ $ p = q $.
Moreover, if $ \nu $ can be sampled from in a tractable manner, then $
\operatorname{GFD}_{\nu}(p, q)$ admits a tractable estimator given samples $ \{ Z^{i}
\}_{i=1}^{n}  $ from $ \nu $ of the form
\begin{equation*}
\widehat{ \operatorname{GFD}_{\nu}(p, q)} = \frac{1}{n} \sum_{i=1}^{n} \| s_p(Z^{i}) - s_q(Z^{i})  \|^2.
\end{equation*}
In practice, the tractability assumption as well as the support assumption for
any $ p $, $ q $ are verified by setting $ \nu $ to be a standard Gaussian
distribution. 

\paragraph{The Exponentiated-GFD Kernel}
Importantly, the (square root of the) Generalized Fisher Divergence is a Hilbertian metric on the space of probability densities.
Indeed, for $ p $, $ q $ such that $ \mathbb{E}_{ \nu } \left \| s_p  \right\|^2, \mathbb{E}_{ \nu } \left \| s_q  \right \|^2 < +\infty $, we have that
\begin{equation*}
\operatorname{GFD}_{\nu}(p, q) = \left \| \phi(p) - \phi(q) \right \|_{\mathcal  L_2(\nu)}^{2}
\end{equation*}
where $ \phi \colon p  \mapsto s_p(\cdot) \in \mathcal L_2(\nu)  $ can be checked to be injective.
The latter fact allows to construct a kernel $ K_{\nu} $ on the space of
probability densities based on the Generalized Fisher Divergence as follows:
\begin{definition}[Exponentiated GFD Kernel]\label{def:expgfdkernel}
    Let $ p $,  $ q $ be two probability densities on $ \mathcal  X $, and
    $ \nu $ a probability measure on $ \mathcal  X $. The \emph{exponentiated
    GFD kernel} between $ p $ and $ q $ is defined as
\begin{equation*}
    K_{\nu}(p, q) \coloneqq e^{ - \operatorname{GFD}_{\nu}(p, q)/ (2\sigma^2)}
\end{equation*}
\end{definition}
Since the (square root of the) $\operatorname{GFD} $ is a Hilbertian metric, $ K_{\nu} $ is positive
definite~\citep{DBLP:conf/icml/MeunierPC22}, and can be estimated given samples
of $ \nu $ by replacing $ \operatorname{GFD}_{\nu} $ with its empirical counterpart. We
summarize the computation method for $ K_{\nu} $ in \cref{alg:expgfdkernel}.

\begin{algorithm2e}
    \DontPrintSemicolon
    \caption{Exponentiated GFD Kernel}\label{alg:expgfdkernel}
    \KwData{Probability densities $p, q$ on $\mathcal{X}$}
	\KwResult{Approx. $\widehat{K_{\nu}(p, q)}$ of $K_{\nu}(p, q)$ in \cref{def:expgfdkernel}}
	\Parameters{Base measure $\nu$, num.\ of base samples $m$}
    \BlankLine
	\For{$i \leftarrow 1$ \KwTo $m$}{
        Draw $Z^i \sim \nu$\;
    }
    \Return{$\exp{\left(- \frac{1}{2m \sigma^2} \sum_{i=1}^{m} \|  s_p(Z^{i}) - s_q(Z^{i})\|^2\right)}$}\;
\end{algorithm2e}

{\bfseries Use in hypothesis testing}
In addition to being tractable to estimate, we show that when $\mathcal X$ is compact (for instance, a bounded subset of $\mathbb R^d$), the exponentiated GFD kernels $K_\nu$ are \emph{universal}. As a consequence,  our KCCSD test, which is an instance of a KCSD test, will be able to distinguish the null-hypothesis from \emph{any} alternative satisfying mild smoothness assumptions, as guaranteed
by \citet[Theorem 1]{jitkrittum2020testing}.
\begin{proposition}\label{prop:exp-gfd-universality}
Assume that $\mathcal X$ is compact, $\nu$ has full support on $\mathcal{X}$, and let $\mathcal{P}_{\mathcal{X}}$ be the set of twice-differentiable probability densities on $\mathcal{X}$ equipped with the norm $\|p\|^2 = \|p\|_{\mathcal L_2(\nu)}^2+ \sum_{i=1}^{d} \|\partial_i p\|^2_{\mathcal L_2(\nu)} + \sum_{i, j=1}^{d} \|\partial_i\partial_j p\|^2_{\mathcal{L}_2(\nu)}$. Then $K_\nu$ is universal for any bounded subset of $\mathcal{P}_{\mathcal{X}}$.
\end{proposition}
\begin{proof}
The proof is given in Section~D.2.
\end{proof}

\subsection{The Kernelized Generalized Fisher Divergence (kernel)}
While the recipe given above suffices to obtain a valid kernel on the space of
probability densities, the approximation error arising from the discretization
of the base measure $\nu$ may scale unfavorably with the dimension of the
underlying space $\mathcal{X}$. To address this issue, it is possible to apply
a kernel-smoothing step to the GFD feature map $ \phi(p) $ by composing it with
an integral operator $ T_{K, \nu} $ associated with a $ \mathcal  X
$-vector-valued kernel $ K $ and its RKHS $ \mathcal  H_{ K} $
\begin{equation*}
T_{K, \nu} \colon f \in \mathcal  L(\mathcal  X, \mathbb{R}^d)  \longmapsto \int_{\mathcal{X}} K_x f(x)  \, \nu(\mathrm{d}x) \in \mathcal  H_{ K}
\end{equation*}
and comparing the difference in feature map using the squared RKHS norm $ \left
\| \cdot \right \|_{\mathcal{H}_K}^{2} $. This choice of feature map yields
another metric, which we call the ``kernelized'' GFD:
\begin{equation*}
\operatorname{KGFD}(p, q) \coloneqq \left \| T_{K, \nu} s_p - T_{k, \nu} s_q \right \|_{\mathcal  H_{K}}^{2}.
\end{equation*}
which, like the GFD, admits a tractable, unbiased estimator:
\begin{equation*}
    \frac{1}{m^2}
    \sum_{i,j=1}^{m}\left \langle  K(Z^{i},
    Z^{j}) (s_p - s_q)(Z^{i}), (s_p - s_q)(Z^{j}) \right \rangle_{\mathcal  X}.
\end{equation*}
Since the KGFD is also a Hilbertian metric, we build upon it to construct our second proposal kernel:
\begin{definition}[Exponentiated KGFD Kernel]\label{def:expkgfdkernel}
    Consider the setting of \cref{def:expgfdkernel}, and let $k$ be a bounded positive definite kernel. The \emph{exponentiated KGFD kernel} is given by:
\begin{equation*}
K_{K, \nu} \coloneqq e^{-\operatorname{KGFD}(p, q)/(2\sigma^2)}
\end{equation*}
\end{definition}
For characteristic kernels $ K $, the integral operator $ T_{K,\nu} $
is a Hilbertian isometry between $ \mathcal  L_2(\nu, \mathbb R^d)$ and $ \mathcal  H_{K} $,
making the exponentiated KGFD kernel positive definite. Additionally, $K_{K, \nu}$ enjoys a
similar universality property as its GFD analogue, as discussed in the next proposition.
\begin{proposition}\label{prop:exp-kgfd-universality}
Assume that $\mathcal X$ is compact, $\nu$ has full-support on $\mathcal{X}$, and let $\mathcal{P}_{\mathcal{X}}$ be the set of twice-differentiable probability densities equipped with the norm $\|p\|^2 = \|p\|_{\mathcal L_2(\nu)}^2+ \sum_{i=1}^{d} \|\partial_i p\|^2_{\mathcal L_2(\nu)} + \sum_{i, j=1}^{d} \|\partial_i\partial_j p\|^2_{\mathcal{L}_2(\nu)}$. Then $K_{K, \nu}$ is universal for any bounded subset of $\mathcal{P}_{\mathcal{X}}$.
\end{proposition}

\paragraph{A diffusion interpretation of the KGFD}
In this section, we establish a relationship between the KGFD and diffusion
processes~\citep{rogers2000diffusions},
further anchoring the KGFD to the
array of previously known divergences while opening the door for possible
refinements and generalizations.
Diffusion processes are well-known instances of
stochastic processes $ (X_t)_{t \geq  0} $ that evolve from some
initial distribution $ \mu_0 $ towards a target distribution $ p $ according to
the differential update rule
\begin{equation*}
    \mathrm{d}X_t = s_p(X_t) \,\mathrm{d}t + \sqrt 2\mathrm{d}W_t, \quad X_0 \sim \mu_0,
\end{equation*}
where $ W_t $ is a standard Brownian motion.
For any time $ t \geq  0 $, the probability density of $ X_t $ is the solution $\mu_{\mu_0, p}(\cdot, t)$
of the so-called Fokker-Planck equation
\begin{equation}\label{eq:fp}
    \frac{\partial \mu(x, t)}{\partial t} = \operatorname{div}(-\mu(x, t) s_p(x)) + \Delta_x \mu(x, t)
\end{equation}
with initial condition $ \mu(\cdot, 0) = \mu_0 $.
\cref{prop:link-kl-gfd} establishes a link between these solutions and the KGFD:

\begin{proposition}[Diffusion interpretation of the KGFD]\label{prop:link-kl-gfd}
    Let $\mu_{\nu, p} $ (resp.\ $ \mu_{\nu, q} $)
    be the solution of \cref{eq:fp} with initial
    condition $ \nu $ and target $ p $ (resp.\ $q$). Let $ k $ be a \emph{real-valued},
    twice-differentiable kernel.
    Then, we have that
    \begin{equation*}
        \lim_{t \to 0} \frac{1}{t}\operatorname{MMD}(\mu_{\nu, p}(\cdot, t), \mu_{\nu, q}(\cdot, t)) = \sqrt { \operatorname{KGFD}(p, q) } 
    \end{equation*}
    where the $ \operatorname{MMD} $ is w.r.t.\ the kernel $ k $, and the $\operatorname{KGFD}$ is with respect to the matrix-valued kernel
    $ \nabla_{x} \nabla_{ y } k(x, y) $.
\end{proposition}
\begin{proof}
    See Section~D of the Appendix.
\end{proof}
\cref{prop:link-kl-gfd} frames the exponentiated KGFD kernel as the $ t  \to 0
$ limit of the kernel obtained by setting
\begin{equation*}
\phi_t \colon p  \longmapsto \nabla_x \log \mu_{\nu, p}(\cdot, t)
\end{equation*}
which is the score of the solution of the Fokker-Planck equation~\cref{eq:fp}
with target $ p $ and initial measure $ \nu $, and setting $ H = \mathcal  H $.
Interestingly, the other limit case $ t  \to \infty $ recovers 
the exponentiated MMD kernel.
Indeed, under mild conditions, the Fokker-Planck solution converges to the
target and thus we have that $ \lim_{  t \to \infty }\phi_t(p) = p  $: the
feature map converges to the identity. Thus, the diffusion framework introduced
above allows to recover both the $ \operatorname{KGFD} $ and the $ \operatorname{MMD} $ as
special cases. However, while the limit $ t  \to 0 $ and $ t  \to \infty $
both yield Hilbertian metrics, it is an open question whether for a given time
$0 < t < \infty$,  $ \phi_t $ is also Hilbertian. A positive answer to this question would
allow to construct positive definite kernels that can possibly overcome the
pitfalls of score-based tools~\citep{wenliang2020blindness, zhang2022towards},
while being computable in finite time.

\section{Fast and scalable calibration tests}\label{sec:fast-calibration-tests}

The framing of the calibration testing problem of \cref{sec:calibration-as-cgof} alongside with the GFD-based kernels of \cref{subsec:expgfd-kernel} allows us to design a fast and scalable alternative to the pioneering tests of \citet{widmann2019calibration}.
The full testing procedure is outlined in \cref{alg:ccgof}.

\begin{algorithm2e}
    \DontPrintSemicolon
    \caption{CGOF Calibration Test (GFD Kernel)}\label{alg:ccgof}
    \KwData{Pairs ${\{(P_{|x^i}, y^i)\}}_{i=1}^n \stackrel{\text{i.i.d.}}{\sim} \mathbb{P}(P_{|X}, Y)$}
    \KwResult{Whether to reject $H_0 \colon \text{``model is calibrated''}$}
	\Parameters{Base measure $\nu$, num.\ of base samples $m$, number of data samples $n$, kernel $l \colon \mathcal{Y}^2 \to \mathbb{R}$, significance level $\alpha$}
    \BlankLine
    \For{$i \leftarrow 1$ \KwTo $m$}{%
        Draw $z^i \sim \nu$\;
    }
    \For{$ 1 \leq i < j \leq n$}{%
        \tcc{Use \cref{alg:expgfdkernel} with base samples $\{z^k\}_{k=1}^m$}
        $\kappa^{i,j} \leftarrow \widehat{K_{\nu}(P_{|x^i}, P_{|x^j})}$\;
    }
    Run \cref{alg:ccgof-tractable} with kernel $k(P_{|x^i}, P_{|x^j}) \coloneqq \kappa^{i,j}$\;
\end{algorithm2e}

{\bfseries Calibration tests as a reliability tests in Bayesian inference}
As one main motivation for studying calibration of generic probabilistic models is Bayesian inference, it is important to note that reliability metrics traditionally
used in Bayesian inference such as conservativeness~\citep{Hermans2021}
differ from the notion of calibration in \cref{eq:calibration}.
We first briefly recall the notion of posterior coverage:
\begin{definition}[Conservativeness of a Bayesian model~\citep{Hermans2021}]
Let $P_{|x}(\cdot)$ be a conditional distribution model for $\mathbb{P}(Y \in \cdot \mid X = x)$, and assume that $P_{|x}$ has a density $f_{P_{|x}}$ for $\mathbb{P}(X)$-almost every $x$.
For level $1 - \alpha \in [0, 1]$, let $\Theta_{P_{|x}}(1 - \alpha)$ be the highest density region of $P_{|x}$.%
\footnote{The highest density region of a probabilistic model $P_{|x}$ with density $f_{P_{|x}}$ is defined~\citep[see, e.g.,][]{Hyndman1996} by
$\Theta_{P_{|x}}(1 - \alpha)\coloneqq \{ y \colon f_{P_{|x}}(y) \geq c_{P_{|x}}(1 - \alpha) \}$ where
$c_{P_{|x}}(1 - \alpha) \coloneqq \sup \{ c \colon \int \mathbbm{1}_{[c, \infty)}(f_{P_{|x}}(y)) \, P_{|x}(\mathrm{d}y) \geq 1 - \alpha \}$.}
Then $P_{|\cdot}$ is said to be \emph{conservative} if
\begin{equation*}
     \E_{(x, y) \sim \mathbb{P}(X, Y)} \mathbbm{1}_{\Theta_{P_{|x}}(1 - \alpha)}(y) \geq 1 - \alpha.
\end{equation*}
\end{definition}
In the following proposition, we show that a probabilistic model that is calibrated according to \cref{eq:calibration} is also conservative in the sense of \citet{Hermans2021}, grounding the use of our tests in Bayesian inference. 
\begin{proposition}[Calibrated models are conservative]
If a model $P_{|\cdot}$ is calibrated in the sense of \cref{eq:calibration}, then it is conservative.
\end{proposition}
The proof is given in Section~C of the appendix.

\section{Experiments}\label{sec:experiments}

\begin{figure*}[htb]
    \centering
    \includegraphics[width=\linewidth]{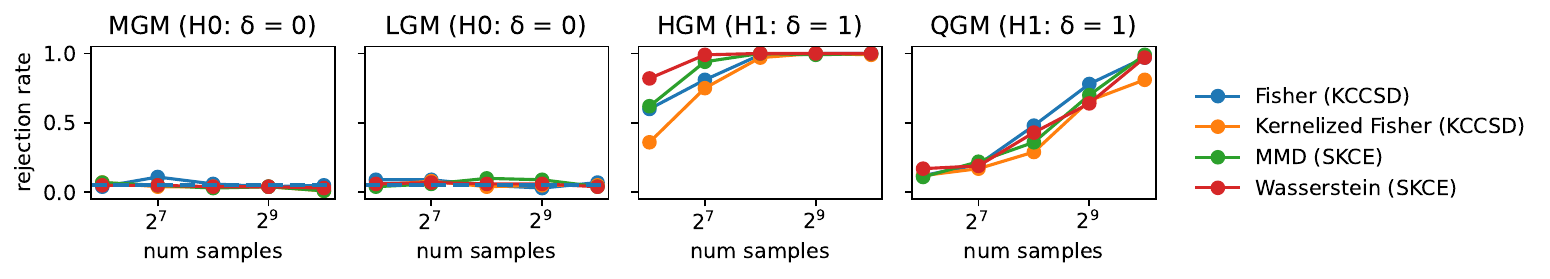}
    \caption{Rejection rates of the KCCSD and SKCE tests with a Gaussian kernel on the target space $\mathcal{Y}$ (significance level $\alpha = 0.05$). All kernels and test statistics are evaluated exactly using closed-form expressions.}
    \label{fig:comparison_kccsd_skce_rejection}
\end{figure*}

\begin{figure*}[htb]
    \centering
    \includegraphics[width=\linewidth]{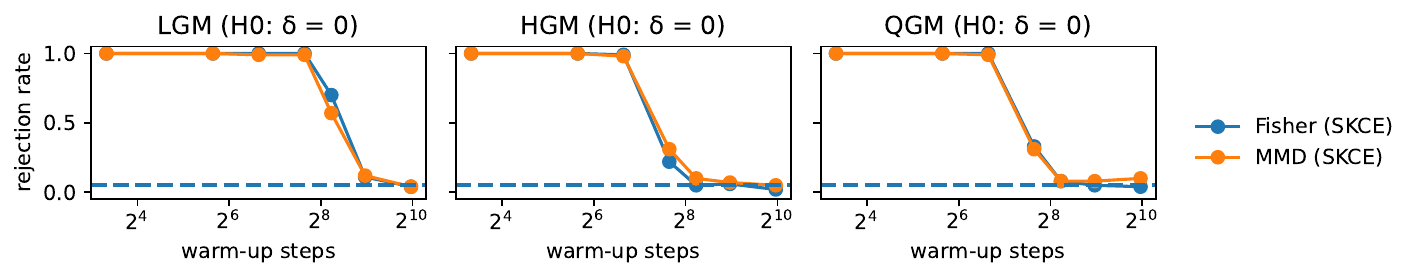}
    \caption{False rejection rates of the SKCE tests for the calibrated LGM, HMC, and QGM ($n = 200$ data points, significance level $\alpha = 0.05$). The expectations in the test statistic are estimated with 2 samples obtained with the Metropolis-adjusted Langevin algorithm (MALA) without step size tuning.}
    \label{fig:skce_mcmc_miscalibrated}
\end{figure*}

We validate the properties of our proposed calibration tests with synthetic data and compare them with existing tests based on the SKCE.%
\footnote{The code to reproduce the experiments is available at \url{https://github.com/pierreglaser/kccsd}.}
More concretely, we run KCCSD tests using either a exponentiated GFD kernel or kernelized exponentiated GFD kernel with a matrix-valued kernel of the form in \cref{eq:kernel_identity} with real-valued Gaussian kernel $k$; and compare them with SKCE tests using two already investigated kernels on distributions: the exponentiated MMD kernel with a Gaussian kernel on the ground space, and, for isotropic Gaussian distributions, the exponentiated Wasserstein kernel with closed-form expression
\begin{multline*}
    k_W\big(\mathcal{N}(\mu, \sigma^2 I_d), \mathcal{N}(\mu', {\sigma'}^2 I_d) \big) \\
    = \exp{\big(- (\|\mu - \mu'\|_2^2 + d (\sigma^2 - {\sigma'}^2)) / (2 \ell^2)\big)}.
\end{multline*}
We set the base measure $\nu$ of the GFD and kernelized GFD kernels to be a standard Gaussian.  On $\mathcal{Y}$, we study the Gaussian and the inverse multi-quadric (IMQ) kernel.

We repeated all experiments with 100 resampled datasets and used a wild bootstrap with 500 samples for approximating the quantiles of the test statistic with a prescribed significance level of $\alpha = 0.05$.
The bandwidths of the kernels are selected with the median heuristic.
A "second-order" median heuristic is used for the ground-space kernels of the KGFD and the exponentiated MMD kernel:
For each pair of distributions, we compute the median distance between samples from an equally weighted mixture of these distributions (numerically for tractable cases such as Gaussian distributions and using samples otherwise), and then the bandwidth of the kernel is set to the median of these evaluations.

We repeatedly generate datasets ${\{(P_{|x^i}, y^i)\}}_{i}$ in a two-step procedure:
First we sample distributions $P_{|x^i}$ and then we draw a corresponding target $y^i$ for each $P_{|x^i}$.
We compare different setups of targets $Y$ and Gaussian distributions $P_{|X}$
with varying degree $\delta \geq 0$ of miscalibration
(models are calibrated for $\delta = 0$ and miscalibrated otherwise):%
\footnote{MGM is adapted from a model used by \citet{widmann2022calibration}, and LGM, HGM, and QGM were used by \citet{jitkrittum2020testing}.}

\paragraph{Mean Gaussian Model (MGM)}
Here $\mathcal{X} = \mathcal{Y} = \mathbb{R}^5$, $\mathbb{P}(X) = \mathcal{N}(0, I_5)$, $\mathbb{P}(Y \,|\, X = x) = \mathcal{N}(x, I_5)$, and $P_{|x} = \mathcal{N}(x + \delta c, I_5)$ for $c \in \{\mathbf{1}_5, e_1\} \subset \mathbb{R}^5$ (miscalibration of all dimensions or only the first one).

\paragraph{Linear Gaussian Model (LGM)}
Here $\mathcal{X} = \mathbb{R}^5$, $\mathcal{Y} = \mathbb{R}$, $\mathbb{P}(X) = \mathcal{N}(0, I_5)$, and $P_{|x} = \mathcal{N}(\delta + \sum_{i=1}^5 ix_i, 1)$.

\paragraph{Heteroscedastic Gaussian Model (HGM)}
Here $\mathcal{X} = \mathbb{R}^3$, $\mathcal{Y} = \mathbb{R}$, $\mathbb{P}(X) = \mathcal{N}(0, I_3)$, $\mathbb{P}(Y \,|\, X = x) = \mathcal{N}(m(x), 1)$, and $P_{|x} = \mathcal{N}(m(x), \sigma^2(x))$ with $m(x) = \sum_{i=1}^3 x_i$ and $\sigma^2(x) = 1 + 10 \delta \exp{(- \|x - c\|^2_2 / (2 \cdot 0.8^2))}$ for $c = 2/3 \,\mathbf{1}_3$.

\paragraph{Quadratic Gaussian Model (QGM)}
Here $\mathcal{X} = \mathcal{Y} = \mathbb{R}$, $\mathbb{P}(X) = \mathcal{U}(-2, 2)$, $\mathbb{P}(Y \,|\, X = x) = \mathcal{N}(0.1 x^2 + x + 1, 1)$, and $P_{|x} = \mathcal{N}(0.1 (1 - \delta) x^2 + x + 1, 1)$.

\Cref{fig:comparison_kccsd_skce_rejection} demonstrates that the proposed KCCSD tests are calibrated: The false rejection rates (type I errors) of the calibrated MGM and LGM do not exceed the set significance level, apart from sampling noise.
Figures~F.1 and F.7 in the suppplementary material confirm empirically that this is the case also when we approximate the Fisher and MMD kernels using samples.

Moreover, we see in \cref{fig:comparison_kccsd_skce_rejection} that for the miscalibrated HGM the SKCE tests exhibit larger rejection rates, and hence test power, than the KCCSD tests in the small sample regime, regardless of the kernel choice.
This specific setting with Gaussian distributions and a Gaussian kernel on the target space $\mathcal{Y}$ is favourable to the SKCE test as both its test statistic, as well as the exponentiated MMD or Wasserstein kernel evaluations are available in closed-form.
In such analytical scenarios we expect the score-based KCCSD tests to perform worse~\citep{wenliang2020blindness,zhang2022towards}.
However, the KCCSD tests present themselves as a practically useful alternative even in this example:
For the miscalibrated HGM their rejection rates are close to 100\% with $\geq 256$ data points,
and for the miscalibrated QGM they show very similar performance as the SKCE tests.
Overall, as expected, we see in \cref{fig:comparison_kccsd_skce_rejection} that for all studied tests rejection rates for the miscalibrated models increases with increasing number of samples.

One main advantage of the KCCSD over the SKCE is that it has first-class support for unnormalized models for which only the score function is available:
In contrast to the SKCE its test statistic only involves scores but no expectations.
In principle, for unnormalized models these expectations in the test statistic of the SKCE can be approximated with, e.g., MCMC sampling.
However, \cref{fig:skce_mcmc_miscalibrated} shows that there is a major caveat:
If the MCMC method is not tuned sufficiently well (e.g., if the chain is too short or the proposal step size is not tuned properly),
it might return biased samples which causes the SKCE tests to be miscalibrated.
On the other hand, increasing the number of MCMC samples increases the computational advantage of the KCCSD even more.

Another difference between the KCCSD and SKCE is highlighted in Figures~F.1 and F.2:
The number of combinations of kernels for which the test statistic can be evaluated exactly is smaller for the SKCE (in these Gaussian examples, it requires Gaussian kernels on the target space).

One limitation of the (kernelized) exponentiated GFD Kernel is that it necessitates setting an additional hyperparameter: the base measure $\nu$, which weights the score differences between its two input distributions $p$ and $q$ at all points of the ground space $\mathcal X$. While our experiments have set $\nu$ to be a Gaussian measure in order to obtain closed-form expressions for Gaussian $p, q$, other choices may be more adequate depending on the problem at hand. For instance, when $p$ and $q$ are posterior models for a given prior $\pi$, we hypothesize that setting $\nu$ to $\pi$ constitutes a better default choice.

\section{Conclusion} \label{conclusion}
In this paper, we introduced the Kernel Calibration Conditional Stein Discrepancy test, a fast and reliable alternative to prior calibration tests for general, density-based probabilistic models, thereby addressing an important need in the Bayesian inference community. In doing so, we introduced kernels for density-based inputs, which we believe are of independent interest and could be used in other domains such as distribution regression~\citep{DBLP:journals/jmlr/SzaboSPG16} or meta-learning~\citep{denevi2020advantage}.
Moreover, while the set of experiments conducted in this paper focused on ``offline'' calibration testing, its low computational cost
opens the door to promising new use cases. One particularly interesting avenue would consist in using the KCCSD test criterion
as a regularizer directly within the training procedure of a probabilistic model, allowing not only to detect miscalibration but also to prevent it in the first place. We look forward to seeing extensions and applications of the tools introduced in this paper.

%

\subsubsection*{Acknowledgements}
    This research was financially supported by the Centre for Interdisciplinary Mathematics (CIM) at Uppsala University, Sweden, by the projects \emph{NewLEADS - New Directions in Learning Dynamical Systems} (contract number: 621-2016-06079) and \emph{Handling Uncertainty in Machine Learning Systems} (contract number: 2020-04122), funded by the Swedish Research Council, by the \emph{Kjell och M{\"a}rta Beijer Foundation}, by the Wallenberg AI, Autonomous Systems and Software Program (WASP) funded by the Knut and Alice Wallenberg Foundation, and by the Excellence Center at Linköping-Lund in Information Technology (ELLIIT).
    Pierre Glaser and Arthur Gretton acknowledge support from the Gatsby Charitable Foundation.

\bibliography{biblio}

\onecolumn 

\appendix

\numberwithin{equation}{section}
\counterwithin{figure}{section}
\part*{Supplementary Material}
\section{Conditional Goodness-of-Fit: General Operator-Valued Kernel}\label{app-sec:cgof-general-kernel}

Assume that
\begin{itemize}
\item kernel $l \in \mathcal{C}^2(\mathcal{Y} \times \mathcal{Y}, \mathbb{R})$,
\item densities $P_{|x} \in C^1(\mathcal{Y}, \mathbb{R})$ for $\mathbb{P}(X)$-almost all $x$, and that
\item $\E_{(x,y) \sim \mathbb{P}(X, Y)} \left\|K_{P_{|x}} \xi_{P_{|x}}(y, \cdot) \right\|_{\mathcal{F}_K} < \infty$.
\end{itemize}
Due to the Bochner integrability of $(x, y) \mapsto K_{P_{|x}} \xi_{P_{|x}}(y, \cdot)$ expectation and inner product commute~\citep[see][Definition~A.5.20]{Steinwart2008SVM}, and hence we have
\begin{equation*}
\begin{split}
    C_{P_{|\cdot}}(\mathbb{P}) &= \left \| \E_{(x, y) \sim \mathbb{P}(X, Y)}\left [ K_{P_{|x}} \xi_{P_{|x}}(y, \cdot) \right ]  \right \|^{2}_{\mathcal  F_{K}} \\
    &= \bigg\langle \E_{(x, y) \sim \mathbb{P}(X, Y)}\left [ K_{P_{|x}} \xi_{P_{|x}}(y, \cdot) \right ], \E_{(x', y') \sim \mathbb{P}(X, Y)}\left [ K_{P_{|x'}} \xi_{P_{|x'}}(y', \cdot) \right ] \bigg\rangle_{\mathcal  F_{K}} \\
    &= \E_{(x, y) \sim \mathbb{P}(X, Y)} \E_{(x', y') \sim \mathbb{P}(X, Y)} \bigg\langle K_{P_{|x}} \xi_{P_{|x}}(y, \cdot), K_{P_{|x'}} \xi_{P_{|x'}}(y', \cdot) \bigg\rangle_{\mathcal  F_{K}} \\
    &= \E_{(x, y) \sim \mathbb{P}(X, Y)} \E_{(x', y') \sim \mathbb{P}(X, Y)} \bigg\langle K^*_{P_{|x'}} K_{P_{|x}} \xi_{P_{|x}}(y, \cdot), \xi_{P_{|x'}}(y', \cdot) \bigg\rangle_{\mathcal  F^{d_y}_l},
\end{split}
\end{equation*}
where $K^*_{P_{|x'}}$ is the adjoint of $K_{P_{|x'}}$.
The reproducing property implies $K^*_{P_{|x'}} K_{P_{|x}} = K(P_{|x}, P_{|x'})$, and therefore we get
\begin{equation*}
\begin{split}
    C_{P_{|\cdot}}(\mathbb{P}) &= \E_{(x, y) \sim \mathbb{P}(X, Y)} \E_{(x', y') \sim \mathbb{P}(X, Y)} \bigg\langle K(P_{|x},P_{|x'}) \xi_{P_{|x}}(y, \cdot), \xi_{P_{|x'}}(y', \cdot) \bigg\rangle_{\mathcal  F^{d_y}_l} \\
    &= \E_{(x, y) \sim \mathbb{P}(X, Y)} \E_{(x', y') \sim \mathbb{P}(X, Y)} H((P_{|x}, y), (P_{|x'}, y'))
\end{split}
\end{equation*}
where
\begin{equation*}
\begin{split}
    H((p, y), (p', y')) &\coloneqq \bigg\langle K(p, p') \xi_{p}(y, \cdot), \xi_{p'}(y', \cdot) \bigg\rangle_{\mathcal  F^{d_y}_l} \\
    &= \bigg\langle K(p, p') \xi_{p}(y, \cdot), l(y', \cdot) \nabla_{y'} \log f_{p'}(y') + \nabla_{y'} l(y', \cdot) \bigg\rangle_{\mathcal  F^{d_y}_l}.
\end{split}
\end{equation*}
For $i \in \{1,\ldots,d_y\}$, let $\operatorname{proj}_i \colon \mathcal{F}_l^{d_y} \to \mathcal{F}_l$ be the projection map to the $i$th subspace of the product space $\mathcal{F}_l^{d_y}$, and similarly let $\iota_i \colon \mathcal{F}_l \to \mathcal{F}_l^{d_y}$ be the embedding of $\mathcal{F}_l$ in the $i$th subspace of $\mathcal{F}_l^{d_y}$ via $x \mapsto (0, \ldots, 0, x, 0, \ldots, 0)$.
Then we can write
\begin{equation*}
\begin{split}
    H((p, y), (p', y')) &= \sum_{i=1}^{d_y} \bigg\langle \operatorname{proj}_i K(p, p') \xi_{p}(y, \cdot), l(y', \cdot) \frac{\partial}{\partial y'_i} \log f_{p'}(y') + \frac{\partial}{\partial y'_i} l(y', \cdot) \bigg\rangle_{\mathcal  F_l} \\
    &= \sum_{i=1}^{d_y} \left[(\operatorname{proj}_i K(p, p') \xi_{p}(y, \cdot))(y') \frac{\partial}{\partial y'_i} \log f_{p'}(y')
    + \frac{\partial}{\partial y'_i} (\operatorname{proj}_i K(p, p') \xi_{p}(y, \cdot))(y')\right].
\end{split}
\end{equation*}
Since $K(p, p') \in \mathcal{L}(\mathcal{F}_l^{d_y})$ is a linear operator, we have
\begin{equation*}
    K(p, p') \xi_p(y, \cdot) = K(p, p') (l(y, \cdot) \nabla_y \log f_p(y)) + K(p, p') \nabla_y l(y, \cdot).
\end{equation*}
For $1 \leq i, j \leq d_y$, define $K_{i,j}(p, p') \colon \mathcal{F}_l \to \mathcal{F}_l$ as the continuous linear operator
\begin{equation*}
    K_{i,j}(p, p') := \operatorname{proj}_i K(p, p') \iota_j.
\end{equation*}
Thus we have
\begin{equation*}
    \operatorname{proj}_i K(p, p') \xi_p(y, \cdot) = \sum_{j=1}^{d_y} \left[\frac{\partial}{\partial y_j} \log f_p(y)\right] K_{i,j}(p, p') l(y, \cdot) + \sum_{j=1}^{d_y} \frac{\partial}{\partial y_j} K_{i,j}(p, p') l(y, \cdot),
\end{equation*}
and therefore
\begin{equation*}
    (\operatorname{proj}_i K(p, p') \xi_p(y, \cdot))(y') =\sum_{j=1}^{d_y}  \left[\frac{\partial}{\partial y_j} \log p(y)\right]  (K_{i,j}(p, p') l(y, \cdot))(y') + \sum_{j=1}^{d_y} \frac{\partial}{\partial y_j} (K_{i,j}(p, p') l(y, \cdot))(y').
\end{equation*}
Due to the differentiability of kernel $l$ we can interchange inner product and differentiation~\citep[][Lemma~4.34]{Steinwart2008SVM}, and thus we obtain
\begin{equation*}
\begin{split}
    H((p, y), (p', y')) ={}& \sum_{i,j=1}^{d_y} \left[\frac{\partial}{\partial y_j} \log f_p(y) \right]\left[\frac{\partial}{\partial y'_i} \log f_{p'}(y') \right]
    (K_{i,j}(p, p') l(y, \cdot))(y') \\
    &+ \sum_{i,j=1}^{d_y}\left[ \frac{\partial}{\partial y'_i} \log f_{p'}(y')\right] \frac{\partial}{\partial y_j} (K_{i,j}(p, p') l(y, \cdot))(y') \\
    &+ \sum_{i,j=1}^{d_y} \left[\frac{\partial}{\partial y_j} \log f_p(y) \right] \frac{\partial}{\partial y'_i} (K_{i,j}(p, p') l(y, \cdot))(y') \\
    &+ \sum_{i,j=1}^{d_y} \frac{\partial}{\partial y'_i}  \frac{\partial}{\partial y_j} (K_{i,j}(p, p') l(y, \cdot))(y'),
\end{split}
\end{equation*}
Define $A \colon (P_{|\mathcal{X}} \times \mathcal{Y})^2  \to \mathbb{R}^{d_y \times d_y}$ by
\begin{equation*}
    [A((p, y), (p', y'))]_{i,j} := (K_{i,j}(p, p') l(y, \cdot))(y') \qquad (1 \leq i, j \leq d_y).
\end{equation*}
Thus we obtain
\begin{equation}\label{eq:calibration-gof-test-statistics-advanced}
    H((p, y), (p', y')) = (s_{p'}(y') + \nabla_{y'})^{\top} A((p, y), (p', y')) (s_{p}(y) + \nabla_{y}),
\end{equation}
where for $x, x' \in \mathbb{R}^d, M(x, x') \in \mathbb{R}^{d \times d}$ we use the notation
\begin{equation*}
\nabla_x^\top M(x, x') = \begin{bmatrix}
\nabla_x^\top [M(x, x')]_{:,1} & \cdots & \nabla_x^\top [M(x, x')]_{:,d}
\end{bmatrix}
=\begin{bmatrix}
\operatorname{div}_x [M(x, x')]_{:,1} & \cdots & \operatorname{div}_x [M(x, x')]_{:,d}
\end{bmatrix},
\end{equation*}
and similarly
\begin{equation*}
M(x, x') \nabla_{x'} = {\left(\nabla_{x'}^\top M(x, x')^\top\right)}^\top =
\begin{bmatrix}
\operatorname{div}_{x'} [M(x, x')]_{1,:} &
\cdots &
\operatorname{div}_{x'} [M(x, x')]_{d,:}
\end{bmatrix}^\top
\end{equation*}
and
\begin{equation*}
\nabla_x^\top M(x, x') \nabla_{x'} =
\nabla_x^\top (M(x, x') \nabla_{x'}^\top) = 
\sum_{i,j=1}^d \frac{\partial^2}{\partial x_i \partial x'_j} {[M(x, x')]}_{i,j}.
\end{equation*}

Thus, given samples $\{(P_{|x^i}, y^i)\}_{i=1}^n \stackrel{\text{i.i.d.}}{\sim} \mathbb{P}(P_{|X}, Y)$, an unbiased estimator of statistic $C_{P_{|\cdot}}(\mathbb{P})$ is
\begin{equation*}
    \widehat{C_{P_{|\cdot}}} = \frac{2}{n(n-1)} \sum_{1 \leq i < j \leq n} H((P_{|x^i}, y^i), (P_{|x^j}, y^j)),
\end{equation*}
where $H$ is given by \cref{eq:calibration-gof-test-statistics-advanced}.

If kernel $K$ is of the form in \cref{eq:kernel_identity}, we recover the simpler formula in \cref{eq:calibration-cgof-test-statistic}.
In this case $A((p, y), (p', y')) = k(p, p') l(y, y') I_{d_y} \in \mathbb{R}^{d_y \times d_y}$, i.e., $A$ is a scaled identity matrix.

\section{KCCSD as a special case of SKCE}\label{app-sec:kccsd-relation-skce}

We prove the following general lemma that establishes the KCSD as a special case of the MMD.
Then \cref{prop:kccsd-relation-skce} follows immediately by considering random variables $Z = P_{|X}$ and $Y$, and models $Q_{|z} = z = P_{|x}$.

\begin{lemma}[KCSD as a special case of the MMD]\label{lemma:kcsd-relation-mmd}
Let $Q_{|z}$ be models of the conditional distributions $\mathbb{P}(Y \in \cdot \,|\, Z = z)$.
Moreover, we assume that
\begin{itemize}
    \item $Q_{|z}$ has a density $f_{Q_{|z}} \in C^1(\mathcal{Y}, \mathbb{R})$ for $\mathbb{P}(Z)$-almost all $z$,
    \item kernel $l \in C^2(\mathcal{Y} \times \mathcal{Y}, \mathbb{R})$,
    \item $\E_{(z,y) \sim \mathbb{P}(Z, Y)} \left\|K_{z} \xi_{Q_{|z}}(y, \cdot) \right\|_{\mathcal{F}_K} < \infty$, and
    \item $\oint_{\partial \mathcal{Y}} l(y, y') f_{Q_{|z}}(y) n(y) \, \mathrm{d}S(y') = 0$ and $\oint_{\partial \mathcal{Y}} \nabla_{y} l(y, y') f_{Q_{|z}}(y') n(y') \, \mathrm{d}S(y') = 0$ for $\mathbb{P}(Z)$-almost all $z$,
\end{itemize}
where $n(y)$ is the unit vector normal to the boundary $\partial \mathcal{Y}$ of $\mathcal{Y}$ at $y \in \mathcal{Y}$.%
\footnote{These assumptions are not restrictive in practice since they are satisfied
if the conditions of \cite[Theorem~1]{jitkrittum2020testing} hold
which are required to ensure that $D_{Q_{|\cdot}}(\mathbb{P}) = 0$ if and only if $Q_{|Z}(\cdot) = \mathbb{P}(Y \in \cdot | Z)$ $\mathbb{P}(Z)$-almost surely.}

Then
\begin{equation*}
    D_{Q_{|\cdot}}(\mathbb{P}) = \operatorname{MMD}_{k_{Q_{|\cdot}}}^2(\mathbb{P}(Z, Y), \mathbb{P}_{Q_{|\cdot}}(Z, Y) )
\end{equation*}
where we define distribution $\mathbb{P}_{Q_{|\cdot}}$ by
\begin{equation*}
    \mathbb{P}_{Q_{|\cdot}}(Z \in A, Y \in B) := \int_A Q_{|z}(Y \in B) \, \mathbb{P}(Z \in \mathrm{d}z)
\end{equation*}
and kernel $k_{Q_{|\cdot}} \colon (\mathcal{Z} \times \mathcal{Y}) \times (\mathcal{Z} \times \mathcal{Y}) \to \mathbb{R}$ as
\begin{equation*}
    k_{Q_{|\cdot}}((z, y), (z', y')) := (s_{Q_{|z'}}(y') + \nabla_{y'})^\mathsf{T} A((z, y), (z', y')) (s_{Q_{|z}}(y) + \nabla_y),
\end{equation*}
using the same notation as in \cref{app-sec:cgof-general-kernel} and similarly defining $A((z, y), (z', y')) \in \mathbb{R}^{d_y \times d_y}$ by
\begin{equation*}
    \left[A((z, y), (z', y'))\right]_{i,j} := (K_{i,j}(z, z') l(y, \cdot))(y') \qquad (1 \leq i, j \leq d_y).
\end{equation*}
If $K$ is of the form $k(\cdot, \cdot) I_{\mathcal{F}_l^{d_y}}$, function $A$ simplifies to
\begin{equation*}
    A((z, y), (z', y')) = k(z, z') l(y, y') I_{d_y}
\end{equation*}
and kernel $k_{Q_{|\cdot}}$ is given by
\begin{multline*}
k_{Q_{|\cdot}}((z, y), (z', y')) \\
= k(z, z') \left[ l(y, y') s_{Q_{|z}}(y)^\mathsf{T} s_{Q_{|z'}}(y') + s_{Q_{|z}}(y)^\mathsf{T}\nabla_{y'} l(y, y') + s_{Q_{|z'}}(y')^\mathsf{T} \nabla_y l(y, y') + \sum_{i=1}^{d_y} \frac{\partial^2}{\partial y_i \partial y'_i} l(y, y')\right].
\end{multline*}
\end{lemma}

\begin{proof}
From a similar calculation as in \cref{app-sec:cgof-general-kernel}~\cite[cf.][Section~A.2]{jitkrittum2020testing} we obtain that
\begin{equation*}
    k_{Q_{|\cdot}}((z, y), (z', y')) = \bigg\langle K_z \xi_{Q_{|z}}(y, \cdot), K_{z'} \xi_{Q_{|z'}}(y', \cdot) \bigg\rangle_{\mathcal{F}_K}.
\end{equation*}

Thus $k_{Q_{|\cdot}}$ is an inner product of the features of $(z, y)$ and $(z', y')$ given by the feature map $(z, y) \mapsto K_z \xi_{Q_{|z}}(y, \cdot) \in \mathcal{F}_K$,
and therefore $k_{Q_{|\cdot}}$ is a positive-definite kernel.
Moreover, from our assumption we obtain
\begin{equation*}
    \E_{(z, y) \sim \mathbb{P}(Z, Y)} {|k_{Q_{|\cdot}}((z, y), (z, y))|}^{1/2} = \E_{(z, y) \sim \mathbb{P}(Z, Y)} \left\|K_{z} \xi_{Q_{|z}}(y, \cdot) \right\|_{\mathcal{F}_K} < \infty.
\end{equation*}
Thus the mean embedding $\mu_{\mathbb{P}(Z, Y)} \in \mathcal{F}_K$ of $\mathbb{P}(Z, Y)$ exists~\citep[][Lemma~3]{gretton2012kernel}.

Due to the Bochner integrability of $(z, y) \mapsto K_{z} \xi_{Q_{|z}}(y, \cdot)$ expectation and inner product commute~\citep[see][Definition~A.5.20]{Steinwart2008SVM}, and hence we have
\begin{equation*}
\begin{split}
\E_{(z, y) \sim \mathbb{P}_{Q_{|\cdot}}(Z, Y)} \E_{(z', y') \sim \mathbb{P}_{Q_{|\cdot}}(Z, Y)} k_{Q_{|\cdot}}((z, y), (z', y'))
&= \left \|\E_{(z, y) \sim \mathbb{P}_{Q_{|\cdot}}(Z, Y)} K_{z} \xi_{Q_{|z}}(y, \cdot) \right\|^2_{\mathcal{F}_K} \\
&= \left \|\E_{z \sim \mathbb{P}(Z)}\E_{y \sim Q_{|z}} K_{z} \xi_{Q_{|z}}(y, \cdot) \right\|^2_{\mathcal{F}_K} \\
&= \left \|\E_{z \sim \mathbb{P}(Z)} K_z \E_{y \sim Q_{|z}} \xi_{Q_{|z}}(y, \cdot) \right\|^2_{\mathcal{F}_K}.
\end{split}
\end{equation*}
Due to the last assumption~\citep[][Lemma~5.1]{Chwialkowski16KGOF} we know that
\begin{equation*}
\E_{y \sim Q_{|z}} \xi_{Q_{|z}}(y, \cdot) = 0,
\end{equation*}
which implies
\begin{equation*}
    \E_{(z, y) \sim \mathbb{P}_{Q_{|\cdot}}(Z, Y)} \E_{(z', y') \sim \mathbb{P}_{Q_{|\cdot}}(Z, Y)} k_{Q_{|\cdot}}((z, y), (z', y')) = 0.
\end{equation*}
Thus the mean embedding $\mu_{\mathbb{P}_{Q_{|\cdot}}(Z, Y)} \in \mathcal{F}_K$ of $\mathbb{P}_{Q_{|\cdot}(Z, Y)}$ exists and satisfies $\|\mu_{\mathbb{P}_{Q_{|\cdot}}(Z, Y)}\|^2_{\mathcal{F}_K} = 0$, and hence $\mu_{\mathbb{P}_{Q|\cdot}(Z, Y)} = 0$.
We obtain~\citep[][Lemma~4]{gretton2012kernel} that
\begin{equation*}
\begin{split}
\operatorname{MMD}^2_{k_{Q_{|\cdot}}}(\mathbb{P}(Z, Y), \mathbb{P}_{Q_{|\cdot}}(Z, Y)) &= \|\mu_{\mathbb{P}(Z, Y)} - \mu_{\mathbb{P}_{Q_{|\cdot}}(Z, Y)}\|^2_{\mathcal{F}_K} \\
&= \|\mu_{\mathbb{P}(Z, Y)}\|^2_{\mathcal{F}_K} \\
&= \E_{(z, y) \sim \mathbb{P}(Z, Y)} \E_{(z', y') \sim \mathbb{P}(Z, Y)} k_{Q_{|\cdot}}((z, y), (z', y')) \\
&= \E_{(z, y) \sim \mathbb{P}(Z, Y)} \E_{(z', y') \sim \mathbb{P}(Z, Y)} \bigg\langle K_z \xi_{Q_{|z}}(y, \cdot), K_{z'} \xi_{Q_{|z'}}(y', \cdot) \bigg\rangle_{\mathcal{F}_K} \\
&= D_{Q_{|\cdot}}(\mathbb{P}),
\end{split}
\end{equation*}
where the last equality follows from \cite[][Section~A.2]{jitkrittum2020testing}.
\end{proof}

\section{Calibration implies expected coverage}

We show that the sense of calibration employed by our tests implies posterior coverage in the sense of \citet{Hermans2021}.
Again let us note $P_{|x}(\cdot)$ for a model of the conditional distribution $\mathbb{P}(Y \in \cdot \mid X = x)$.
Moreover, we assume that $P_{|x}$ has a density $f_{P_{|x}}$ for $\mathbb{P}(X)$-almost every $x$.

For level $1 - \alpha \in [0, 1]$, let $\Theta_{P_{|x}}(1 - \alpha)$ be the highest density region of a probabilistic model $P_{|x}$ with density $f_{P_{|}}$.
It is defined~\citep[see, e.g.,][]{Hyndman1996} by
\begin{equation*}
    \Theta_{P_{|x}}(1 - \alpha) \coloneqq \left\{ y \colon f_{P_{|x}}(y) \geq c_{P_{|x}}(1 - \alpha) \right\}
\end{equation*}
where
\begin{equation*}
    c_{P_{|x}}(1 - \alpha) := \sup \left\{ c \colon \int_{\left\{\tilde{y} \colon f_{P_{|x}}(\tilde{y}) \geq c \right\}} \, P_{|x}(\mathrm{d}y) \geq 1 - \alpha \right\}.
\end{equation*}
Hence, by definition \citep[see, e.g.,][]{Hermans2021}
\begin{equation*}
    \E_{y \sim P_{|x}} \mathbbm{1}\big\{y \in \Theta_{P_{|x}}(1 - \alpha)\big\}
    =\int_{\Theta_{P_{|x}}(1 - \alpha)} \, P_{|x}(\mathrm{d}y) \geq 1 - \alpha.
\end{equation*}

Assume that model $P_{|\cdot}$ is calibrated.
By definition, it satisfies
\begin{equation*}
    \mathbb{P}(Y \in \cdot \mid P_{|X}) = P_{|X} \qquad \mathbb{P}(X)\text{-almost surely}.
\end{equation*}
Hence, for all $\alpha \in [0,1]$, we obtain
\begin{equation*}
\begin{split}
     \E_{(x, y) \sim \mathbb{P}(X, Y)} \mathbbm{1}\big\{y \in \Theta_{P_{|x}}(1 - \alpha)\big\} 
     &= \E_{(P_{|x}, y) \sim \mathbb{P}(P_{|X}, Y)} \mathbbm{1}\big\{y \in \Theta_{P_{|x}}(1 - \alpha)\big\} \\
     &= \E_{P_{|x} \sim \mathbb{P}(P_{|X})} \E_{y \sim P_{|x}} \mathbbm{1}\big\{y \in \Theta_{P_{|x}}(1 - \alpha)\big\} \\
     &\geq \E_{P_{|x} \sim \mathbb{P}(P_{|X})} \big[1 - \alpha \big] \\
     &= 1 - \alpha.
\end{split}
\end{equation*}
Thus model $P_{|\cdot}$ has expected coverage for all $\alpha \in [0, 1]$.

\section{Diffusion-Limit and Universality}
\subsection{Fisher divergence as a diffusion limit}\label{app-sec:limit-fisher-divergence}

We recall that for a map $ f $  and a measure  $ \mu $, the push-forward measure of $ \mu $ by $ f $, noted $ f_{\#} \mu $, 
is the measure on the image space of $ f $ which verifies, for any measurable function $ g $
\begin{equation*}
    \int_{  }^{  } g(x) \, f_{\#} \mu(\mathrm{d}x) = \int_{  }^{  } g(f(x)) \, \mu(\mathrm{d}x).
\end{equation*}

To prove the differential inequality linking the MMD and the KGFD, we rely on the following reformulation of the Fokker-Planck equation:
\begin{equation*}
\begin{split}
    \frac{\partial \mu(x, t)}{\partial t} &= \operatorname{div}_x(-\mu(x, t) s_p(x)) + \Delta_x \mu(x, t) \\
					  &= \operatorname{div}_x(-\mu(x, t) s_p(x)) + \operatorname{div}_x \nabla_{ x } \mu(x, t) \\
					  &= \operatorname{div}_x(-\mu(x, t) s_p(x)) + \operatorname{div}_x(\mu(x, t) \nabla_x \log \mu(x, t)) \\
					  &= \operatorname{div}_x(- \mu(x, t)(s_p(x) - \nabla_x \log \mu(x, t)).
\end{split}
\end{equation*}
We remark that since the density $\mu(x, t)$ is twice differentiable in $ x $ and differentiable in $ t $~\citep{johnson2004information}, this equation holds in the strong sense, and not only in the sense of distributions. Because of that, one has
\begin{equation*}
    \partial_t \mu(x, t) = \lim_{ \Delta \to 0 } \frac{\mu(x, t + \Delta) - \mu(x, t)}{\Delta}.
\end{equation*}

Let us consider an RKHS $ \mathcal  H $ with kernel $ k $, and let  $ h \in \mathcal  H $.
Let us define $m_t(x) := m(x, t) := \mu_{\nu,p}(x, t) - \mu_{\nu,q}(x, t) $ and we note $ \operatorname{MMD}(m_t)$ the function given by
\begin{equation*}
\operatorname{MMD}(m_t) = \left[\iint k(x, y) m_t(x) m_t(y) \, \mathrm{d}x \, \mathrm{d}y\right]^{1/2} = \operatorname{MMD}(\mu_{\nu,p}(\cdot, t), \mu_{\nu,q}(\cdot, t)).
\end{equation*}
To show that $ \lim_{ t  \to 0 } \frac{ d }{ \text{d}t }\operatorname{MMD}(m_t) = \operatorname{KGFD}(p, q) $,
we first analyze the differential properties of the easier to handle $\operatorname{MMD}^2$ and complete the proof using a chain rule argument.
The first variation (also called Gateaux Derivative) of $m  \mapsto \operatorname{MMD}^2(m)$ is a linear functional
on the space of functions
\begin{equation*}
    \left\{f - g \,\middle|\, f, g \colon \mathcal{X} \times [0, \infty) \to \mathbb{R} \quad \text{with} \quad \forall t \geq 0 \colon \int_{\mathcal{X}} f(x, t) \,\mathrm{d}x = \int_{\mathcal{X}} g(x, t) \,\mathrm{d}x = 1 \right\},
\end{equation*} given by
\begin{equation*}
\frac{ \delta \operatorname{MMD}^2 }{ \delta m} \colon f \mapsto \int 2 k(x, y) m_t(x) f(y) \,\mathrm{d}x \,\mathrm{d}y.
\end{equation*}
Using the chain rule for Gateaux derivatives, we have that
\begin{equation*}
\begin{split}
    \frac{\mathrm{d} \operatorname{MMD}^2(m)}{\mathrm{d} t} &= \frac{ \mathrm{d} \operatorname{MMD}^2}{ \mathrm{d} m}(m) \frac{ \mathrm{d} m}{ \mathrm{d} t}\\ 
	&= \int 2 k(x,y) m_t(x) \frac{\mathrm{d} m}{\mathrm{d} t}(y) \,\mathrm{d}x \,\mathrm{d}y.
\end{split}
\end{equation*}
From the Fokker-Planck Equation, we have that
\begin{equation*}
\begin{split}
    \frac{\mathrm{d}m }{\mathrm{d}t} &= \partial_t \mu_{\nu, p} - \partial_t \mu_{\nu, q} \\
				    &= \operatorname{div}_x (\mu_{\nu, p} \nabla_x \log \frac{p}{ \mu_{\nu, p} }) - \operatorname{div}_x (\mu_{\nu, q} \nabla_x \log \frac{q}{ \mu_{\nu, q} }) \\
				    &= \operatorname{div}_x (\nu \nabla_x \log \frac{ p }{ \nu}) - \operatorname{div}_x (\nu \nabla_x \log \frac{ q }{ \nu }) + o(1) \\
                   &= \operatorname{div}_x (\nu \nabla_x \log \frac{p}{q}) + o(1)
\end{split}
\end{equation*}
Plugging the last equation in the chain rule, we have:
\begin{equation*}
\begin{split}
	\frac{\mathrm{d} \operatorname{MMD}^2(m)}{\mathrm{d}t}					&= \int_{  }^{  } 2 m_t(x) \text{div}_y \nu(y) \nabla_{ y } \log \frac{ p }{ q }(y) k(x,y) \text{d}x \text{d}y + o(1) \\
						& = \int_{  }^{  } 2m_t(x) \left \langle \nabla_{ y }  k(x, y), \nu(y) \nabla_{ y } \log \frac{ p }{ q }(y) \right \rangle \,\mathrm{d}x \,\mathrm{d}y + o(1).
\end{split}
\end{equation*}
Similarly, since $ m_0 = \mu_{\nu, p}(\cdot, 0) - \mu_{\nu, q}(\cdot, 0) = \nu - \nu = 0 $, we have $ m_t(x) = t \partial_t m(x, 0) + o_x(t)$.
The calculation follows as:
\begin{equation*}
\begin{split}
    \frac{\mathrm{d} \operatorname{MMD}^2(m)}{ \mathrm{d} t} &= \int 2 t \times \partial_t m(x, t) \left \langle \nabla_{ y }  k(x, y), \nu(y) \nabla_{ y } \log \frac{ p }{ q }(y) \right \rangle \,\mathrm{d}x \,\mathrm{d}y + o(t) \\
	&= \int 2 t \times \operatorname{div}_x \nu(x) \nabla_{ x } \log \frac{ p }{ q }(x) \left \langle \nabla_{ y }  k(x, y), \nu(y) \nabla_{ y } \log \frac{ p }{ q }(y) \right \rangle \,\mathrm{d}x \,\mathrm{d}y  + o(t) \\
	&= \int 2 t \times \left \langle  \nu(x)\nabla_{ x }\log \frac{ p }{ q }(x), \nabla_{ x }  \left \langle \nabla_{ y }  k(x, y), \nu(y) \nabla_{ y } \log \frac{ p }{ q }(y) \right \rangle \right \rangle \,\mathrm{d}x \,\mathrm{d}y  + o(t)\\
	&= \int 2 t \times \left \langle  \nu(x)\nabla_{ x }\log \frac{ p }{ q }(x), \nabla_{ x }  \nabla_{ y }  k(x, y), \nu(y) \nabla_{ y } \log \frac{ p }{ q }(y) \right \rangle \,\mathrm{d}x \,\mathrm{d}y + o(t).
\end{split}
\end{equation*}
To get rid of the degenerate scaling as $ t  \to 0 $, we now focus on (the
derivative of) $ \sqrt {\operatorname{MMD}^2(m_t)}  $ as $ t  \to 0 $. Notice that
since $ \operatorname{MMD}(m_0) = 0 $, the derivative of $ \sqrt {\operatorname{MMD}^2(m_t)}$
does not exist a priori for $ t=0 $: we consider instead  $ \frac{\mathrm{d}}{\mathrm{d}t } \sqrt {\operatorname{MMD}^2(m_t)}\Big|_{t=t}  $, and extend it by continuity by setting $
t  \to 0 $. We have:
\begin{equation*}
\begin{split}
\frac{\mathrm{d}\sqrt {\operatorname{MMD}^2(m_t)}  }{ \mathrm{d}t } = \frac{1}{2 \sqrt {\operatorname{MMD}^2(m_t)} } \frac{ \mathrm{d} \operatorname{MMD}^2(m_t)}{ \mathrm{d} t}.
\end{split}
\end{equation*}
As
\begin{equation*}
\operatorname{MMD}^2(m_t) = \int_{  }^{  } k(x, y) m_t(x) m_t(y) \,\mathrm{d}x \,\mathrm{d}y
\end{equation*}
we obtain through similar calculations that
\begin{equation*}
    \operatorname{MMD}^2(m_t) = \int_{  }^{  } \int_{  }^{  } t^2\left \langle  \nu(x)\nabla_{ x }\log \frac{ p }{ q }(x), \nabla_{ x }  \nabla_{ y }  k(x, y), \nu(y) \nabla_{ y } \log \frac{ p }{ q }(y) \right \rangle \,\mathrm{d}x \,\mathrm{d}y + o(t)
\end{equation*}
from which the results follows. Note that the matrix-valued kernel $ (K(x,
y))_{ij} = (\nabla_{x}  \nabla_{y}  k(x, y))_{ij}$ is positive
definite, a result akin to one of \citet{zhou2008derivative} but for the matrix-valued case.
Indeed, for all $ x, y \in
\mathcal  X $, $ z, t \in  \mathbb{R}^d $, 
\begin{equation*}
    z K(x, y)t = \left\langle \sum\limits_{i=1}^{d} z_i \partial_i k(x, \cdot), \sum\limits_{i=1}^{d} t_i \partial_i k(y, \cdot) \right \rangle_{\mathcal H} 
\end{equation*}
where $ \partial_i k(x, \cdot) \in \mathcal  H $~\citep{zhou2008derivative}.
In the following, we write $ \phi(x, y) = \sum_{i=1}^{d} y_i \partial_i
k(x_i, \cdot) $. Now, for all sets of $ \{ x^{i} \}_{i=1}^{n}
\in \mathcal  X  $,
$\{ y^{j} \}_{i=1}^{n} \in  \mathbb{R}^d$, we have
\begin{equation*}
\begin{split}
    \sum\limits_{ i, j=1 }^{ n } \left \langle K(x^i, x^j) y^j, y^i \right \rangle_{\mathbb{R}^d}
    &= \sum\limits_{ i, j=1 }^{ n } \left \langle \phi(x^i, y^i), \phi(x^j, y^j)  \right \rangle_{\mathcal  H} \\ 
    &=  \left \langle \sum\limits_{ i=1 }^{  n} \phi(x^i, y^i), \sum\limits_{ i=1 }^{ n } \phi(x^i, y^i) \right \rangle_{\mathcal  H}  \geq  0
\end{split}
\end{equation*}
from which it follows that $ K $ is indeed positive definite~\citep[Theorem
2.1]{micchelli2005learning}.

\subsection{Universality of the Exponentiated-GFD and Exponentiated-KGFD kernel}
To prove the universality of $K_\nu$ and $K_{\nu, K}$ under the assumptions discussed in the related propositions,
we rely on the following theorem \cite[Theorem 2.2]{christmann2010universal}.
\begin{theorem}
On a compact metric space $(\mathcal Z, d_\mathcal Z$ ) and for a continuous
and injective map $\phi : \mathcal Z \mapsto H$,  where H is a separable
Hilbert space, the kernel
$K(z, z') = e^{-\gamma \| \phi(z) - \phi(z')\|^2_H}$ is universal.
\end{theorem}
We first focus on the universality of $K_\nu$.
We set as our goal to apply that theorem to our setting, in which $\mathcal Z := \mathcal P_\mathcal X$ 
is a (sub)set of probability densities, which needs to be associated with a suitably chosen metric
in order to make $\mathcal P_\mathcal X$ to be compact, and $\phi$ continuous.
As bounded subsets of differentiable densities, whose elements can be framed as elements
of the Sobolev space of first order $\mathcal W^{2, 1}(\nu)$~\citep{taylor1996partial}), are not compact a priori,
we restrict ourselves to twice-differentiable densities with bounded Sobolev norm of second order, i.e., to $\mathcal W^{2,2 }(\nu)$ with norm $\|p\|^2_{\mathcal W^{2, 2}} \coloneqq \|p\|_{\mathcal L_2(\nu)}^2+ \sum_{i=1}^{d} \|\partial_i p\|^2_{\mathcal L_2(\nu)} + \sum_{i, j=1}^{d} \|\partial_i\partial_j p\|^2_{\mathcal{L}_2(\nu)}$.
From the Rellich-Kondrachov theorem~\citep{taylor1996partial}, 
we know that when $\nu$ has compact support, the canonical canonical injection $I \colon \mathcal W^{2,2}(\nu) \to \mathcal W^{2,1}(\nu)$ is a compact operator.
As a consequence, for any bounded subset $A$ of $\mathcal P_{\mathcal X}$ we thus have that $I(A)$ is compact for $\|f\|_{\mathcal W^{2, 1}}^2:=\|f\|_{\mathcal L_2(\nu)}^2+ \sum_{i=1}^{d} \|\partial_i f\|^2_{\mathcal L_2(\nu)} $, which implies that any bounded subset $A$ of $P_\mathcal X$ is compact for $d(z, z') = \|z - z'\|_{\mathcal W^{2, 1}}$.
To apply the above theorem, it remains to prove the continuity and injectivity of $\phi: p \mapsto \nabla \log p$ under this metric (in that case the separable Hilbert space $H$ is set to $\mathcal L_2(\nu)$). And indeed, for such a choice of $d$, $\phi$ and $H$, $\phi$ is continuous. To prove this fact, remark that differentiable densities with full support on $\mathcal X$ are bounded away from $0$, making the use of a $\phi \colon p \mapsto \nabla \log p = \nabla p / p$ continuous.
Moreover, $\phi$ is injective as $d_{W^{2, 1}}(p, q) \coloneqq \|p - q\|_{W^{2, 1}} \neq 0$ implies $\|\nabla \log p  - \nabla \log q\|_{\mathcal L_2(\nu)} \neq 0$.
Thus, all conditions of \cite[Theorem 2.2]{christmann2010universal} are satisfied, and the result follows as a consequence.

We now move on to prove the universality of $K_{\nu, K}$.
The proof follows the same reasoning as the proof of the universality of $K_\nu$, the only difference being the fact that the feature map $\tilde \phi$ of $K_{K, \nu}$ is given by $T_\nu \circ \phi$, where $\phi \colon p \mapsto \nabla \log p$ and $T_{K, \nu} \colon \mathcal{L}(\mathcal{X}, \mathbb{R}^d) \to  \mathcal{H}_{K}$ is given by
\begin{equation*}\label{eq:k-int-op}
T_{K, \nu} \colon f \mapsto \int_{\mathcal{X}} K_x f(x)  \, \nu(\mathrm{d}x).
\end{equation*}
However, if $\nu$ is a probability measure and $K$ is bounded, then $T_{K, \nu}$ is a bounded operator, and thus continuous, making $\tilde \phi$ continuous.
Moreover, if $K$ is characteristic, $T_{K, \nu}$ is injective.
Thus $\tilde \phi$ is injective and continuous, from which the result follows by \cite{christmann2010universal}.

\section{Background on Stein and Fisher divergences} \label{app-sec:background-divergences}

\paragraph{The Fisher Divergence} 
Consider two continuously differentiable densities $p$ and $q$ on $\mathbb{R}^d$.
Then the Fisher divergence~\citep{sriperumbudur2017density,johnson2004information} between $p$ and $q$ is defined as:
\begin{equation*}
\operatorname{FD}(p||q) = \int_{ \mathbb{R}^d }^{  } \left \| \nabla_{  } \log p(x) - \nabla_{  } \log q(x) \right \|_{2}^{2} p(x) \,\mathrm{d}x.
\end{equation*}
We refer to \cite{sriperumbudur2017density} for an overview of the properties of the Fisher divergence, including its relative strength w.r.t.\ other divergences, and other formulations. The Fisher divergence was used for learning statistical models of some training data in \cite{hyvarinen2005estimation, sriperumbudur2017density}, and more recently in \cite{song2019generative}.

\paragraph{Stein Discrepancies} 
Of proximity to the Fisher divergence is the family of Stein discrepancies~\citep{anastasiou2022stein}.
Stein discrepancies build upon the concept of Stein operators, which are operators $\mathcal{A}_{\mathbb{P}}$ such that
\begin{equation*} 
\E_{\mathbb{Q}}\left[ \mathcal  A_{\mathbb P} f \right ] = 0 \iff \mathbb{ Q } = \mathbb{ P }
\end{equation*}
for any $ f $ within a set $ \mathcal  G(\mathcal  A_{\mathbb{ P }}) \subset \operatorname{dom}(\mathcal A_{\mathbb P}) $ called the \emph{Stein class}
of $ \mathcal  A_{\mathbb P} $.
Following this definition, the $ \mathcal  A_{\mathbb{ P }} $-stein discrepancy is defined as
\begin{equation*} 
\operatorname{SD}_{\mathcal  A_{\mathbb{ P }}}(\mathbb{ P }, \mathbb{ Q }) = \sup_{ f \in G(\mathcal  A) }\left \|\E_{ \mathbb{ Q } }  \mathcal  A f \right \|
\end{equation*}
which satisfies by construction the axioms of a \emph{dissimilarity} (or \emph{divergence}) measure between $ \mathbb{ P } $ and $ \mathbb{ Q } $.

\paragraph{Link Between the Fisher divergence and Diffusion Stein Discrepancies}
Perhaps the most famous Stein discrepancy is the one that sets $\mathcal  {A}_{\mathbb{P}}$ to be the infinitesimal generator of the isotropic diffusion process toward $\mathbb{ P } $~\citep{gorham2019measuring}:
\begin{equation*}
    \begin{cases}
    \mathrm{d}X_t &= \nabla_{  } \log p(X_t) \,\mathrm{d}t + \sqrt{2}\,\mathrm{d}W_t \\
    (\mathcal  A_{d, \mathbb{P}}f)(\cdot)  &= \left \langle \nabla_{  }  \log p(\cdot), \nabla_{  }  f   \right \rangle  + \left \langle \nabla_{  } , \nabla_{  } f   \right \rangle 
    \end{cases} 
\end{equation*}
Recalling that $\E_{ \mathbb{   P} }\left [ \mathcal A_{d, \mathbb{ P }} f \right ]  = 0$ for all $f \in \mathcal  G(\mathcal  A_{d, \mathbb{ P }}) $, we obtain the following formulation for the diffusion Stein discrepancy
\begin{equation*}
\begin{split}
    \operatorname{SD}_{\mathcal  A_{d, \mathbb{P}}}(\mathbb{ P }, \mathbb{ Q }) &\coloneqq \sup_{  f } \left \| \E_{  \mathbb{ Q } } \mathcal  A_{d, \mathbb{ P }} f \right \| = \sup_{  f } \left \| \E_{  \mathbb{ Q } } (\nabla_{  }   \log p - \nabla_{  }   \log q)^{\top} \nabla_{  }  f  \right \| \\
							    &= \sup_{g =  \nabla_{  } f} \left \| \E_{  \mathbb{ Q } } (\nabla_{  }   \log p - \nabla_{  }   \log q)^{\top} g  \right \|,
\end{split}
\end{equation*}
highlighting the connection between the Fisher divergence and the diffusion Stein discrepancy.

\paragraph{Link Between the Fisher divergence and the Kernelized Stein Discrepancy}
Given a RKHS $\mathcal{H}$ such that $B_{\mathcal{H}^{\otimes d}}(0_{\mathcal{H}^{\otimes d}}, 1)$ is a Stein class for $\mathcal{A}_{d, \mathbb{P}}$, the kernelized Stein discrepancy~\citep{gorham2017measuring} is given by
\begin{equation*}\label{eq:KSD}
\begin{split}
    \operatorname{KSD}(\mathbb{ P }, \mathbb{ Q }) &\coloneqq \sup_{ h = \nabla_{  } f  \in \mathcal  H^{\otimes d}: \left \|h \right \|_{\mathcal H^{\otimes d}} \leq 1 } \left \| \E_{  \mathbb{ Q } } \left \langle \nabla_{  }   \log p(x) - \nabla_{  }   \log q(x), h(x)\right \rangle \right \|\\
					     &=  \sup_{ h = \nabla_{  } f  \in \mathcal  H^{\otimes d}: \left \|h \right \|_{\mathcal H^{\otimes d}} \leq 1 } \left \langle h, \E_{ \mathbb{ Q } } (\nabla_{  }  \log p(x) - \nabla_{  }  \log q(x)  )k(x, \cdot)\right \rangle_{\mathcal H^{\otimes d}}^{1/2}  \\
					     &=  \left \| \E_{ \mathbb{ Q } } \left \lbrack  (\nabla_{  }  \log p(x) - \nabla_{  }  \log q(x)  ) k(x, \cdot) \right \rbrack  \right \|_{\mathcal  H^{\otimes d}} \\
					     &= \left \| I^\star_{k, \mathbb{ Q }} (\nabla_{  }  \log p - \nabla_{  }  \log q  ) \right \|_{\mathcal  H^{\otimes d}}
\end{split}
\end{equation*}
where $I^\star_{k, \mathbb{ Q }}$ is the adjoint of the canonical injection from $ \mathcal  H^{\otimes d} $ to $ (L^2(\mathbb{ Q }))^{\otimes d} $, also known as the \emph{kernel integral operator}.
This derivation shows that the KSD can be seen as a kernelized version of the Fisher divergence.

\paragraph{Link between MMD and KSD} 
It is possible~\citep{gorham2017measuring} to reframe the KSD as an MMD with a specific kernel.
Indeed, given some base kernel $k(x, y)$, define the following ``Stein'' kernel
\begin{equation*}\label{eq:stein-kernel}
    \tilde{ k }(x, y) = \left \langle \nabla_{  }  \log p(x)k(x, \cdot) + \nabla_{  }  k(x, \cdot), \nabla_{  }  \log p(y)k(y, \cdot) + \nabla_{  }  \log k(y, \cdot)     \right \rangle_{\mathcal  H^{\otimes d}}
\end{equation*}
which is positive definite as an inner product of a feature map of $x$.
Then $\mathcal  H_{ \tilde{ k}} = \mathcal A_{d, \mathbb{ P }}( \mathcal  H) $ and $ \left \| f \right \|_{ \mathcal  H_{\tilde{k}}} = \left \| \mathcal  A f \right \|_{\mathcal  H_{k}^{\otimes d}} $.
Moreover, we have that $ \E_{ \mathbb{ P } } \, \tilde{ h} = 0$ for all $\tilde{h} \in \mathcal  H_{ \tilde{k}} $.
By the definition of the KSD, we have that
\begin{equation*} 
\begin{split}
\operatorname{KSD}(\mathbb{ P }, \mathbb{ Q }) &= \sup_{ h   \in \mathcal  H^{\otimes d} \colon \left \|h \right \|_{\mathcal H^{\otimes d}} \leq 1 } \left \| \E_{ \mathbb{ Q } } \mathcal  A_{d, \mathbb{ P }} h  \right \|\\
				       &= \sup_{ h \in \mathcal  H^{\otimes d} \colon \left \|h \right \|_{\mathcal H^{\otimes d}} \leq 1 } \left \| \E_{ \mathbb{ Q } } \mathcal  A_{d, \mathbb{ P }} h  - \E_{  \mathbb{ P } } \mathcal  A_{d, \mathbb{ P }} h\right \|_{\mathcal  H} \\
				       &= \sup_{ h \in \mathcal  H_{\tilde{k}} \colon \left \|h \right \|_{\mathcal  H_{\tilde{k}}} \leq 1 } \left \| \E_{ \mathbb{ Q } }   h  - \E_{  \mathbb{ P } } h\right \|_{\mathcal  H_{\tilde{k}}} \\
				       &= \operatorname{MMD}_{\tilde{ k}}(\mathbb{ P }, \mathbb{ Q }).
\end{split}
\end{equation*}

\paragraph{Differential Inequalities between the KL and the Fisher Divergence} 
It is well known~\citep{carrillo2003kinetic} that the KL divergence can be related to the Fisher divergence by considering the evolution of $\operatorname{KL}(\mathbb{ P }_t||\mathbb{ Q })$ when $ \mathbb{ P }_t $ evolves according to the Fokker-Planck equation
\begin{equation}\label{eq:fokker-planck}
\partial_t p_t(x) = \operatorname{div} ( p_t(x) (\nabla_{  } \log q_t(x) - \nabla_{  }  \log p_t(x)  )), \quad \mathbb P_0=\mathbb P.
\end{equation}
(Two relevant side notes: for any $t\geq 0$, $\mathbb P_t$ is the law at time $t$ of the Markov process $(X_t)_{t\geq 0}$ such that $X_0 \sim \mathbb P$ and undergoing an isotropic diffusion towards $\mathbb Q$. Moreover, \cref{eq:fokker-planck} is also the Wasserstein gradient flow equation of $\operatorname{KL}(\cdot||\mathbb Q)$ starting from $\mathbb  P$).
Recalling that \cref{eq:fokker-planck} is satisfied in the sense of distributions, and relying on Gateaux-Derivative formulas for Free Energy-type functionals~\citep[see][for more precise statements]{ambrosio2005gradient}, we have:
\begin{equation*}
\begin{split}
    \frac{\mathrm{d}\text{KL}(\mathbb P_t || \mathbb Q) }{\mathrm{d}t} &= \frac{ \partial \operatorname{KL} }{\mathrm{d}\mathbb{ P } }\bigg\rvert_{\mathbb{ P }_t} \frac{ d \mathbb{ P }_t }{\mathrm{d}t } \\
				      &= \int_{ }^{  } \left \langle  \nabla_{  } (\log p_t(x) - \log q_t(x) ), (\nabla_{  }  \log q_t - \nabla_{  }  \log p_t  )\right \rangle  \,\mathrm{d}\mathbb{ P }_t(x) \\
				      &= - \operatorname{FD}(\mathbb{ P }_t, \mathbb{ Q }).
\end{split}
\end{equation*}

\begin{figure}[htbp]
    \centering%
    \includegraphics[width=.7\textwidth]{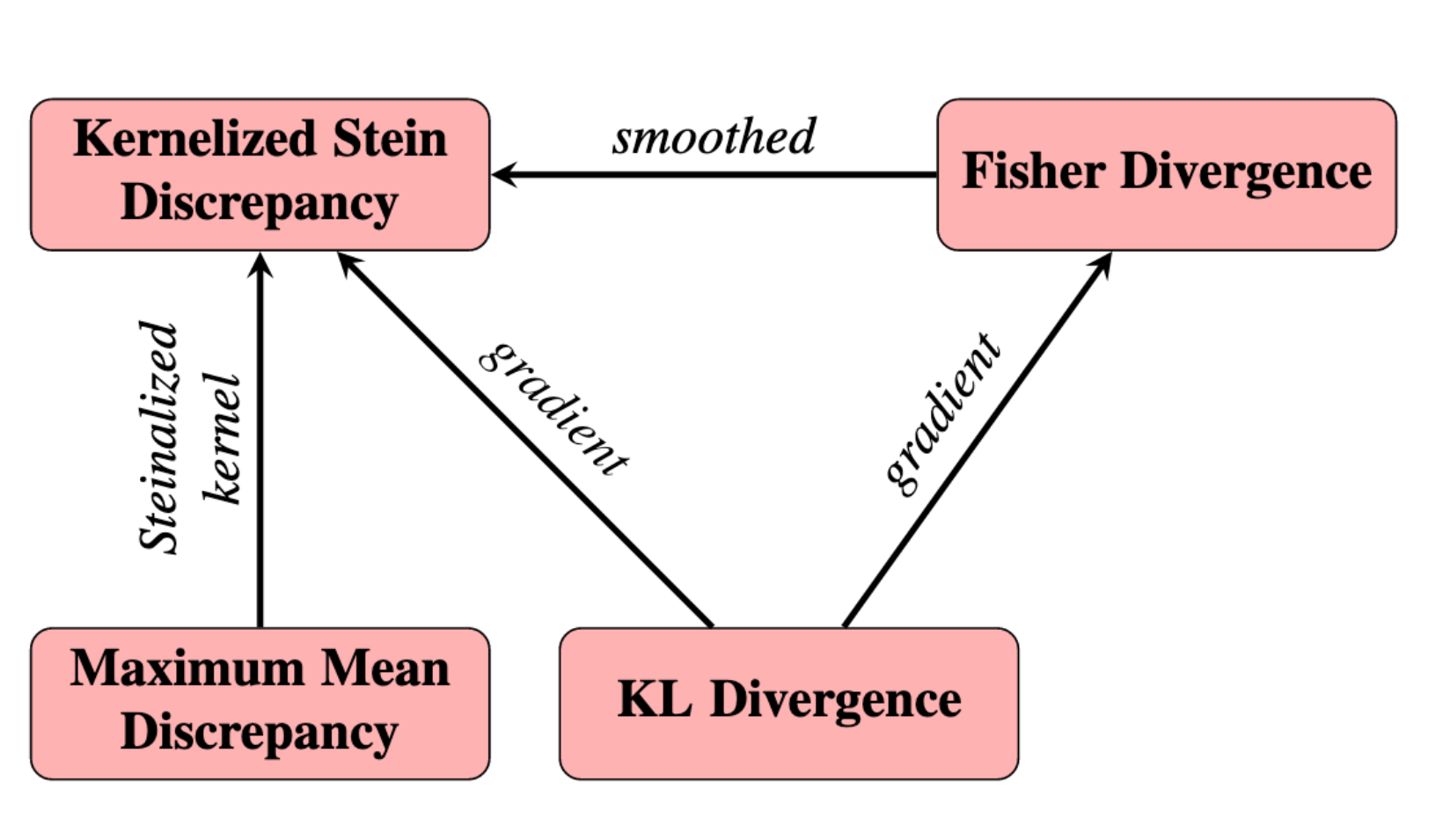} %
    \label{fig:link-divergences}%
    \caption{Relationships between the Fisher divergence, the KL divergence, the MMD, and the KSD~\citep{liu2016short}.}
\end{figure}

\clearpage
\section{Experimental Results}

This section contains visualizations of all experiments discussed in \ref{sec:experiments}, including figures contained in the main text.
In all experiments we set the significance level to $\alpha = 0.05$.
Every experiment is repeated for 100 randomly sampled datasets and with 500 bootstrap iterations for estimating the quantile of the test statistic.

We use Gaussian distributions and compare the KCCSD and the SKCE with different combinations of kernels.
For the KCCSD, for Gaussian distributions all considered test statistics can be evaluated exactly.
Alternatively, for the exponentiated (kernelized) Fisher kernel and the exponentiated MMD kernel one can resort to approximations using samples from the base measure.
For the SKCE, however, the test statistic can be evaluated exactly on in special cases such as Gaussian kernels on the target space.
All approximate evaluations are performed with 10 samples.

\subsection{Mean Gaussian Model}

\begin{figure}[!htb]
    \centering
    \includegraphics[width=\linewidth]{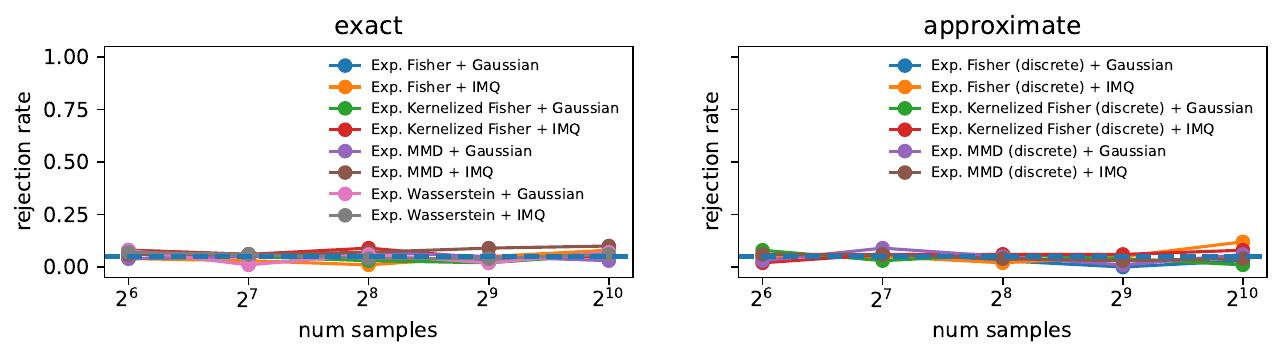}
    \caption{False rejection rate of the KCCSD for MGM ($\delta = 0$).}
\end{figure}

\begin{figure}[!htb]
    \centering
    \includegraphics[width=\linewidth]{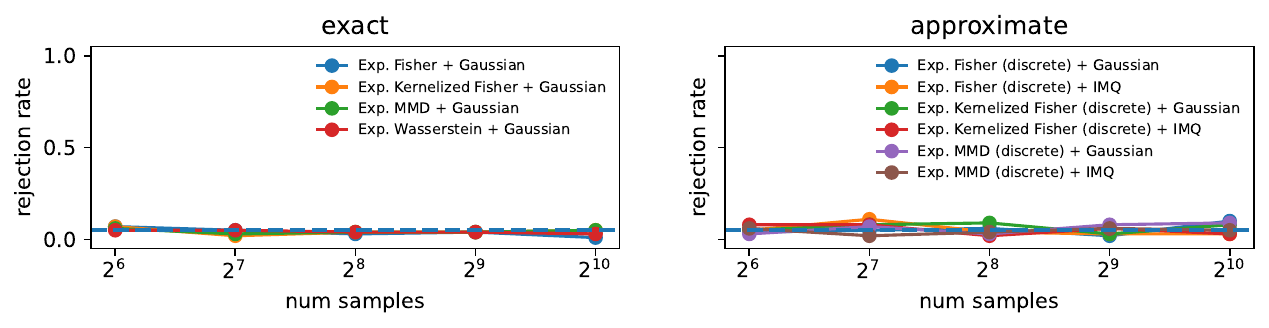}
    \caption{False rejection rate of the SKCE for MGM ($\delta = 0$).}
    \label{fig:mgm_skce}
\end{figure}

\begin{figure}[!htb]
    \centering
    \includegraphics[width=\linewidth]{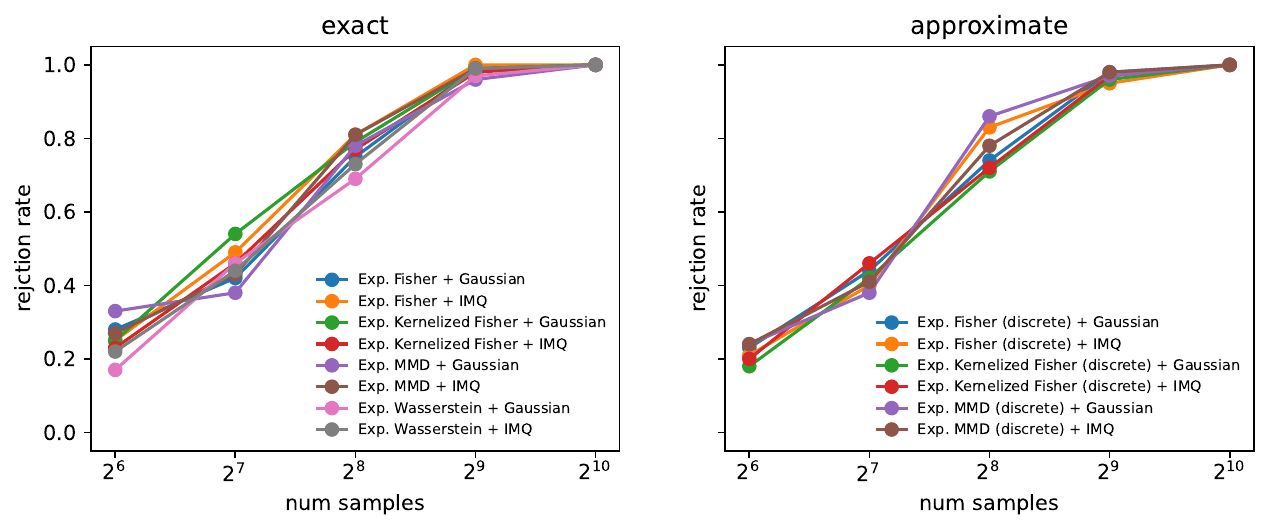}
    \caption{Rejection rate of the KCCSD for MGM ($\delta = 0.1$, $c = \mathbf{1}_d$).}
    \label{fig:pmgm_kccsd_all}
\end{figure}

\begin{figure}[!htb]
    \centering
    \includegraphics[width=\linewidth]{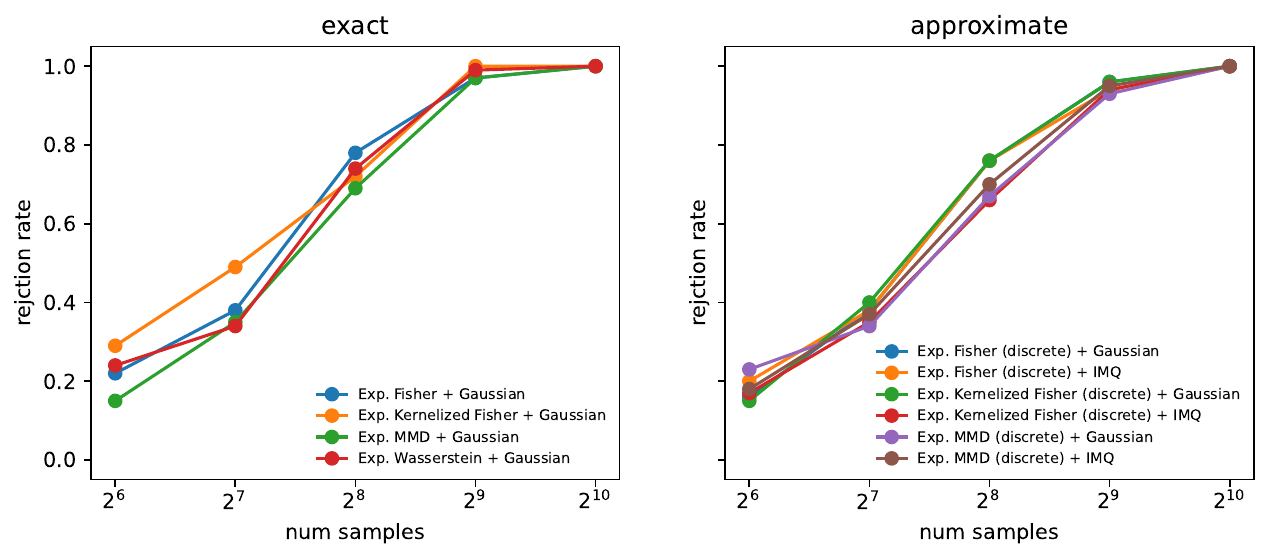}
    \caption{Rejection rate of the SKCE for MGM ($\delta = 0.1$, $c = \mathbf{1}_d$).}
    \label{fig:pmgm_skce_all}
\end{figure}

\begin{figure}[!htb]
    \centering
    \includegraphics[width=\linewidth]{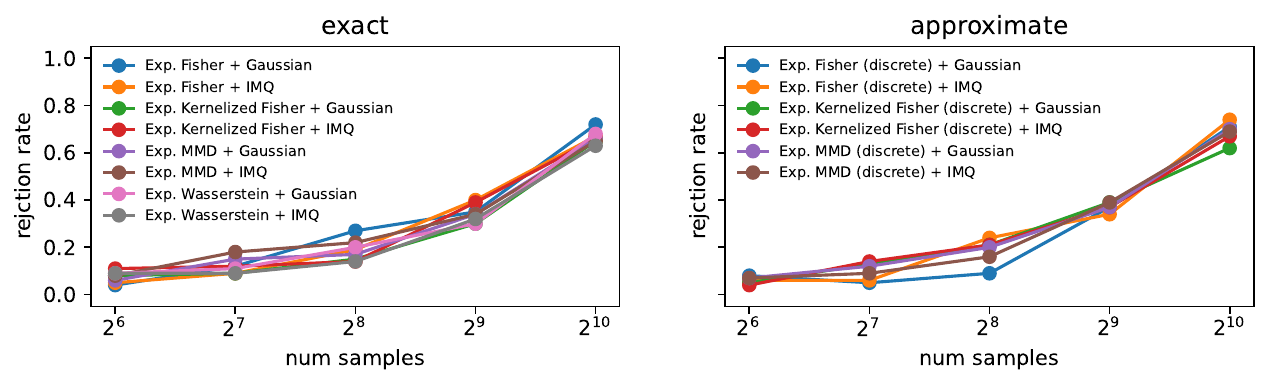}
    \caption{Rejection rate of the KCCSD for MGM ($\delta = 0.1$, $c = e_1$).}
\end{figure}

\begin{figure}[!htb]
    \centering
    \includegraphics[width=\linewidth]{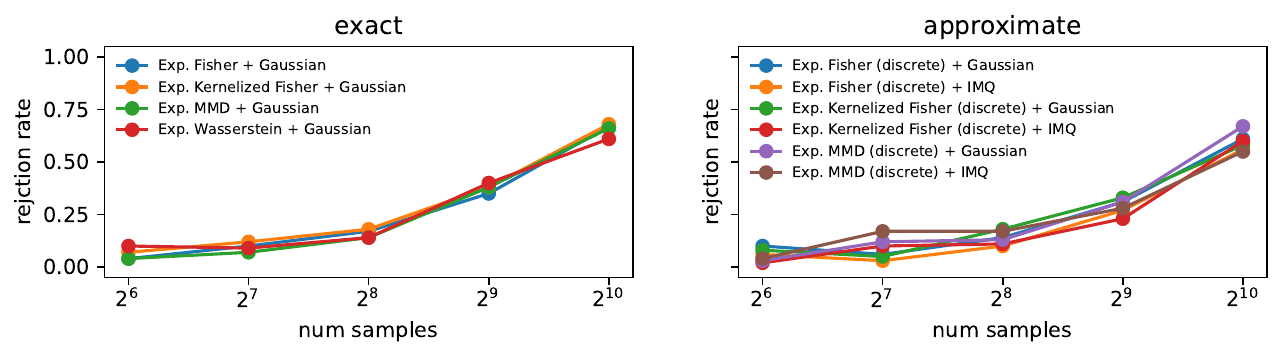}
    \caption{Rejection rate of the SKCE for MGM ($\delta = 0.1$, $c = e_1$).}
    \label{fig:pmgm_skce_first}
\end{figure}

\FloatBarrier
\subsection{Linear Gaussian Model}

\begin{figure}[!htb]
    \centering
    \includegraphics[width=\linewidth]{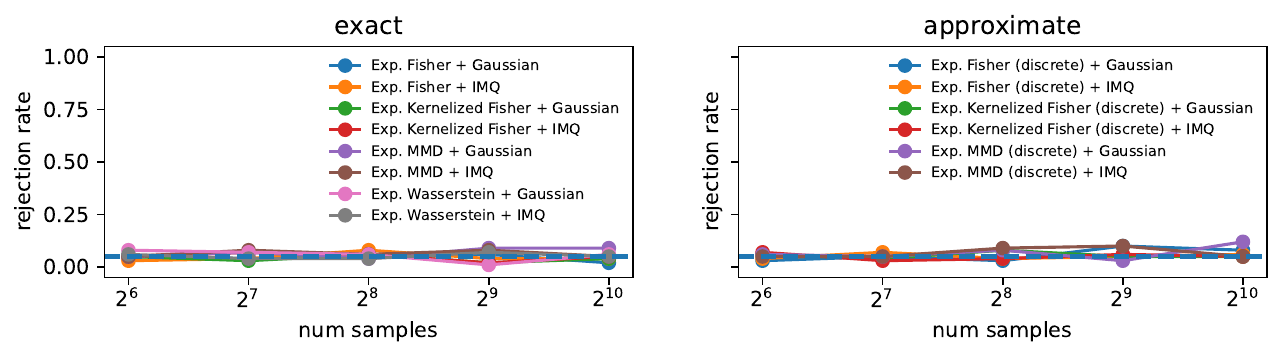}
    \caption{False rejection rate of the KCCSD for LGM ($\delta = 0$).}
    \label{fig:lgm_kccsd}
\end{figure}

\begin{figure}[!htb]
    \centering
    \includegraphics[width=\linewidth]{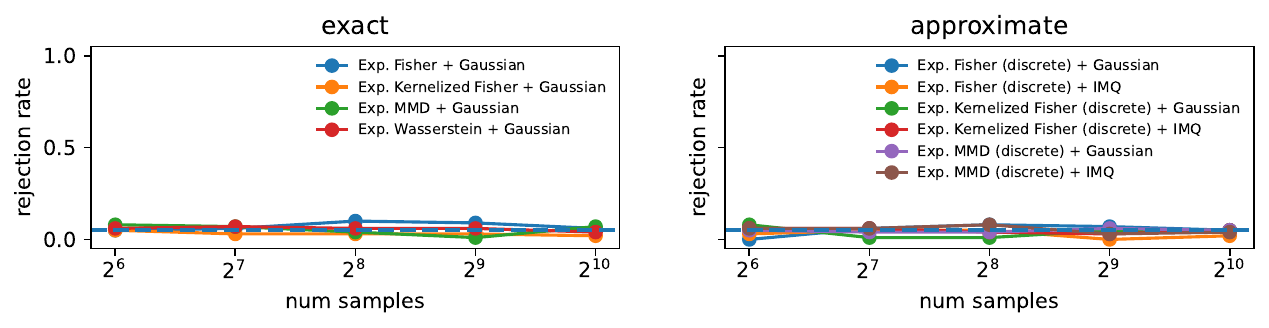}
    \caption{False rejection rate of the SKCE for LGM ($\delta = 0$).}
    \label{fig:lg_skce}
\end{figure}

\FloatBarrier
\subsection{Heteroscedastic Gaussian Model}

\begin{figure}[!htb]
    \centering
    \includegraphics[width=\linewidth]{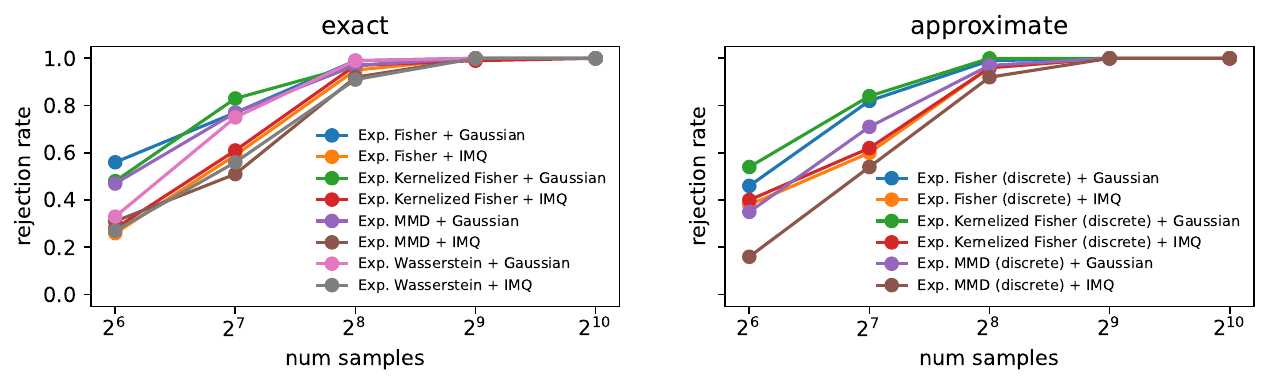}
    \caption{Rejection rate of the KCCSD for HGM ($\delta = 1$).}
    \label{fig:hgm_kccsd}
\end{figure}

\begin{figure}[!htb]
    \centering
    \includegraphics[width=\linewidth]{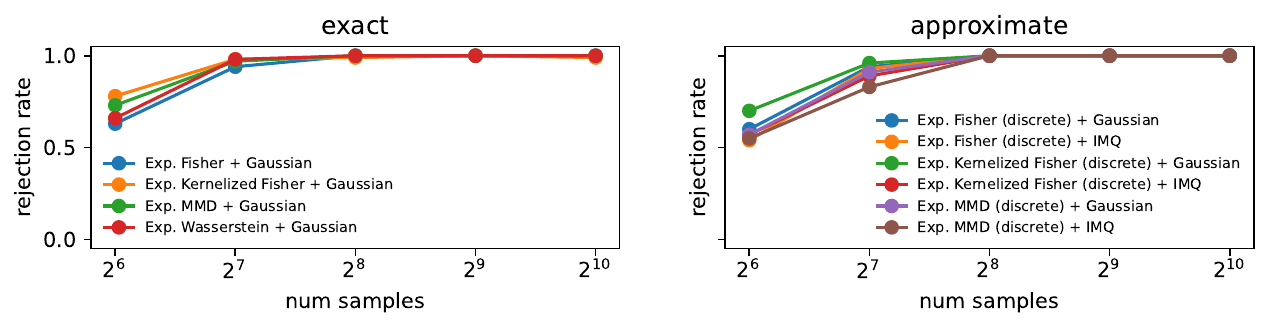}
    \caption{Rejection rate of the SKCE for HGM ($\delta = 1$).}
    \label{fig:hgm_skce}
\end{figure}

\FloatBarrier
\subsection{Quadratic Gaussian Model}

\begin{figure}[!htb]
    \centering
    \includegraphics[width=\linewidth]{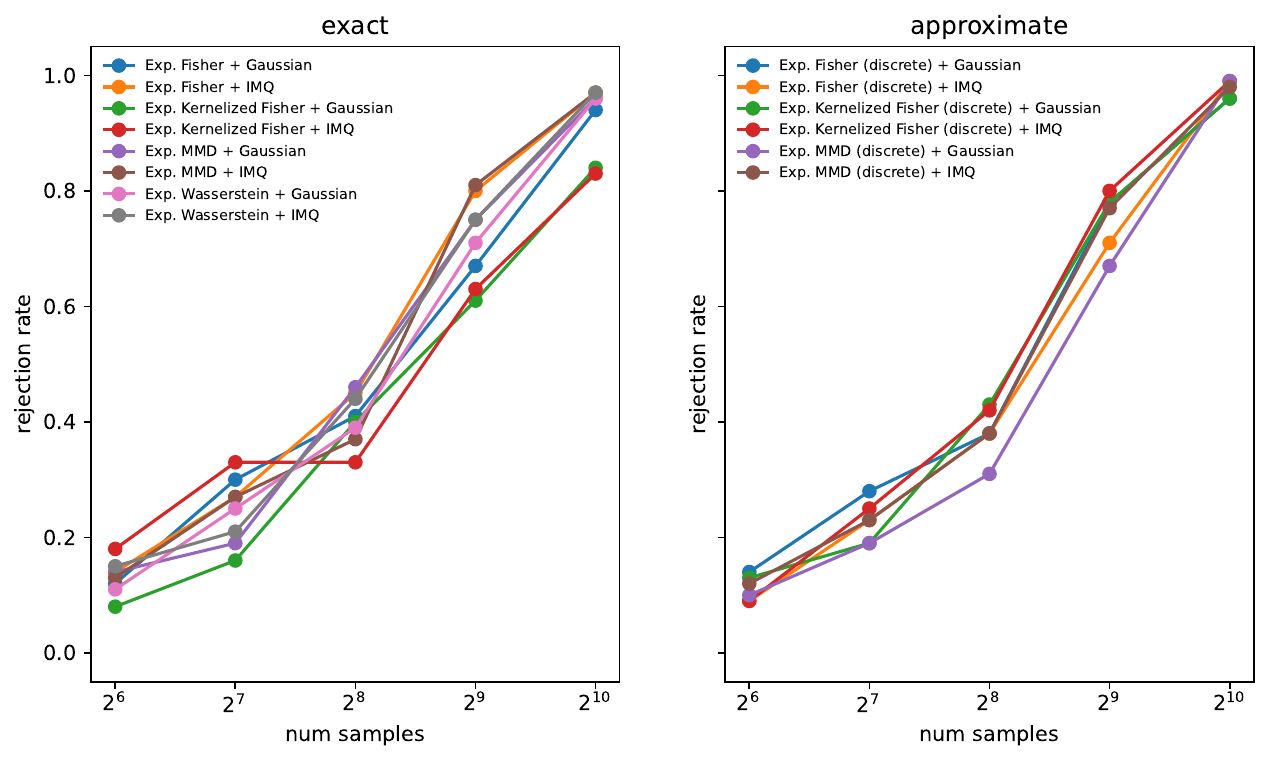}
    \caption{Rejection rate of the KCCSD for QGM ($\delta = 1$).}
    \label{fig:qgm_kccsd}
\end{figure}

\begin{figure}[!htb]
    \centering
    \includegraphics[width=\linewidth]{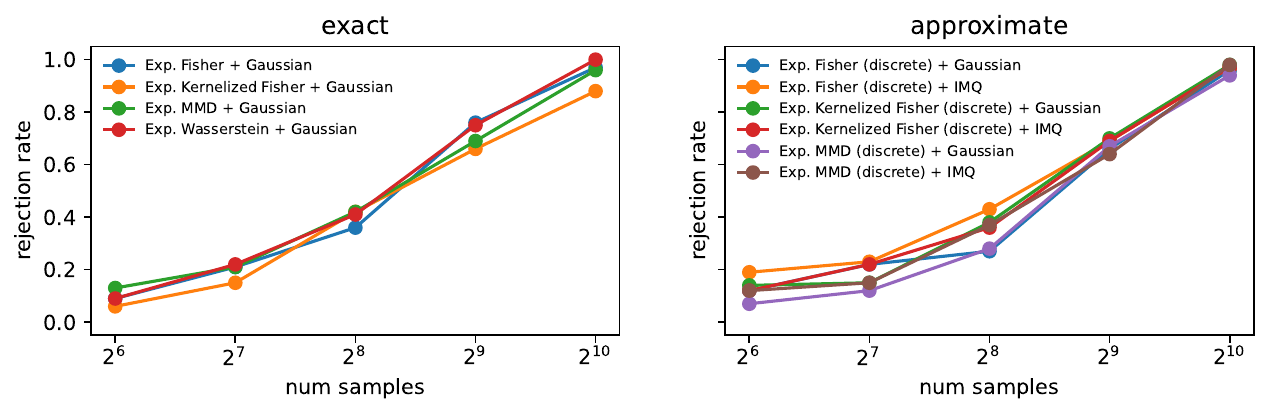}
    \caption{Rejection rate of the SKCE for QGM ($\delta = 1$).}
    \label{fig:qgm_skce}
\end{figure}

\FloatBarrier

\end{document}